\documentclass[11pt]{article}

\usepackage[a4paper,margin=1in]{geometry}
\usepackage{amsmath,amssymb,amsthm,mathtools,bm}
\usepackage{graphicx}
\usepackage{tabularx}
\usepackage{booktabs}
\usepackage[font=small,labelfont=bf]{caption}
\usepackage{subcaption}
\usepackage{float}
\usepackage[section]{placeins}
\usepackage[hidelinks]{hyperref}
\usepackage{cleveref}
\usepackage{enumitem}
\usepackage{microtype}
\usepackage{tikz}
\usetikzlibrary{arrows.meta,positioning,fit,calc}
\usepackage[ruled,vlined]{algorithm2e}
\usepackage[svgnames]{xcolor}
\usepackage{multirow}
\usepackage{cite}

\newtheorem{theorem}{Theorem}[section]
\newtheorem{proposition}[theorem]{Proposition}
\newtheorem{lemma}[theorem]{Lemma}

\theoremstyle{definition}
\newtheorem{definition}[theorem]{Definition}
\theoremstyle{remark}

\newcommand{\D}{\mathrm{D}}

\newcommand{\diag}{\mathrm{diag}}

\newcommand{\Scale}[1]{\mathcal{S}_{#1}}

\title{\bf CSympNet-ID: conformal-symplectic map learning for linearly damped Hamiltonian systems}
\author{
    Jiale Gong$^{1}$ \and
    Pengzhan Jin$^{2}$ \and
    Dongyang Kuang$^{1}$ \and
    Lu Li$^{1}$\thanks{Corresponding author. Email: lilu86@mail.sysu.edu.cn} \and
    Yifa Tang$^{3}$
}
\date{
    {\scriptsize
    $^{1}$School of Mathematics (Zhuhai), Sun Yat-sen University, Zhuhai 519082, China \\
    $^{2}$National Engineering Laboratory for Big Data Analysis and Applications, Peking University, Beijing 100871, China\\
    $^{3}$State Key Laboratory of Mathematical Sciences, Academy of Mathematics and Systems Science, Chinese Academy of Sciences, Beijing 100190, China \\
    }
}

\begin{document}
\maketitle

% ===================== Abstract =====================
\begin{abstract}
Learning dissipative dynamics from discrete observations is essential for reliable long-horizon prediction and physically meaningful parameter identification. For linearly damped Hamiltonian systems, the exact flow is generally not symplectic but conformally symplectic, contracting the canonical symplectic form by a scalar factor that reflects the net dissipation. We propose \emph{Conformal Symplectic Networks with damping identification (CSympNet-ID)}, a discrete-time map-learning framework that learns the one-step flow map directly from snapshot pairs while enforcing exact discrete conformal symplecticity by construction, without penalty terms or projection. The architecture composes an exact symplectic neural core with explicit diagonal scaling layers whose factors are parameterized exponentially by a scalar damping-rate parameter, thereby guaranteeing positivity and interpretability of the learned dissipation factor. We establish a scaling-conjugacy factorization for conformal symplectic maps and derive a pointwise-in-step  density result for CSympNet-ID. We evaluate an irregular-step damped oscillator, a  damped spring-mass chain, a  damped nonlinear cubic oscillator, and additional high-dimensional extensions. CSympNet-ID gives the most favorable overall results among the compared models in the reported experiments, particularly in data-scarce regimes, target contraction-law recovery, and high-dimensional tests where unstructured baselines degrade rapidly.
\end{abstract}

\vspace{0.4em}

\section{Introduction}

Dissipation plays a central role in many real-world dynamical systems, affecting stability, long-horizon predictability, and the identification of physically meaningful parameters. A widely used idealization is a linearly damped Hamiltonian system with a scalar damping rate. 
Although more general models may involve anisotropic, nonlinear, or state-dependent dissipation, a single effective damping coefficient arises naturally in reduced-order, modal, and homogenized descriptions of dissipative physics. 
In mechanical and structural dynamics, proportional or modal damping is routinely used to model vibration decay, and each lightly damped mode can often be represented as an oscillator with velocity-proportional damping \cite{Adhikari2006Damping}. 
A direct example from molecular simulation and sampling is underdamped Langevin dynamics, whose deterministic drift contains the linear friction term, where the scalar coefficient controls relaxation, decorrelation, and sampling efficiency \cite{SkeelHartmann2021}. 
In dynamical-system learning, a broad line of work aims to model unknown
dynamical laws or their evolution maps from observation data
\cite{ChenWuXiu2025DUE}.  For dissipative systems, however, matching short-time trajectories alone is often insufficient: a model that reproduces an incorrect dissipation law may yield wrong decay envelopes, distorted phase portraits, or unstable rollouts when extrapolated beyond the observation window.

A principled route to robust extrapolation is to encode physics and geometry as inductive bias, as emphasized in physics-informed and structure-preserving learning frameworks~\cite{Karniadakis2021PhysicsInformedML,Celledoni2021StructurePreservingDeepLearning}. The structure-aware methods discussed below follow two complementary routes: continuous-time models that learn vector fields or variational structures, and discrete-time models that learn one-step flow maps directly.
At the continuous-time level, neural ordinary differential equations represent dynamics by a learned vector field $\dot z=f_\theta(z,t)$ and rely on numerical integrators to realize the flow \cite{Chen2018NeuralODE};
physics-informed neural networks complement this view by enforcing governing equations through residual losses during training \cite{Raissi2019PINN}, with structure-aware variants developed for multiscale regimes \cite{JinMaWu2023APNN}.
For conservative mechanics, this bias can be further strengthened by Hamiltonian or Lagrangian structure:
Hamiltonian Neural Networks (HNNs) learn a Hamiltonian $H_\theta$ and generate the vector field via canonical equations \cite{Greydanus2019HNN}; related structure-preserving reconstruction methods
approximate the underlying Hamiltonian from trajectory data and then derive the corresponding Hamiltonian system, rather than directly
learning the right-hand side \cite{WuQinXiu2020HamiltonianReconstruction};
and Lagrangian Neural Networks (LNNs) learn a Lagrangian and recover Euler-Lagrange dynamics \cite{Cranmer2020LNN}.
Related continuous-time formulations incorporate symplectic structure directly at the level of neural ODEs and can accommodate control/forcing inputs, e.g., Symplectic ODE-Net \cite{Zhong2020SymODEN}.
From a discrete-time perspective, one may bypass derivative estimation and discretization mismatch by learning the one-step flow map directly \cite{QinWuXiu2019DNNGoverningEquations,DavidMehats2023SymplecticLearningHNN}. 
Symplectic networks such as SympNets pursue this strategy while constraining the learned map to the class of symplectic maps, yielding strong long-term behavior for Hamiltonian systems \cite{Jin2020SympNets};
related symplectic-flow/canonical-transformation architectures provide alternative exact-symplectic parameterizations \cite{Li2020NeuralCanonicalTransformations}.
Moreover, in many applications measurements are naturally available as snapshot pairs $(z_n,z_{n+1})$ at a fixed sampling interval, whereas reliable time derivatives are unavailable or severely noise-amplified; this makes learning the one-step map directly both statistically and computationally attractive.
These developments convey a coherent message: long-horizon fidelity benefits from constraining either the vector field or the flow map to the correct structure class.

For damped systems, however, strict symplecticity and energy conservation are generally not the appropriate targets.
In the scalar linearly damped Hamiltonian setting motivated above, the exact flow is typically conformally symplectic rather than symplectic \cite{McLachlanPerlmutter2001}. In canonical coordinates $z=(q,p)$, the time-$t$ flow map $\Phi_t$ satisfies
\begin{equation}\label{eq:conf_sympl_flow_intro}
\big(\mathrm{D}\Phi_t(z)\big)^\top J\mathrm{D}\Phi_t(z)=\lambda^\star(t)J,
\end{equation}
where $J$ is the canonical symplectic matrix and $\lambda^\star(t)>0$ is the conformal factor encoding the net dissipation. 
Thus the damping rate is not merely a phenomenological decay parameter; it appears geometrically as a scalar  contraction factor $\lambda^\star(t)$ of the canonical symplectic form, which is independent of the phase point for each fixed elapsed time \(t\) but varies with \(t\). Preserving this intrinsic structure is therefore important for reliable long-horizon behavior \cite{HairerLubichWanner2006}.
Conformal symplectic integrators provide discrete-time analogues of \eqref{eq:conf_sympl_flow_intro} and are known to reproduce the correct damping factor under appropriate damping relations \cite{BhattFloydMoore2016,ModinSoderlind2011}; see also
\cite{McLachlanQuispel2002Splitting} for the splitting viewpoint.
Related extensions for damped Hamiltonian PDEs include conformal
multi-symplectic integration methods as well as dissipation-preserving discretizations based on conformal
symplectic formulations \cite{MooreNorenaSchober2013ConformalConservation,CaiZhangWang2017JCP}.

From the learning perspective, existing approaches that incorporate dissipation can be viewed as imposing additional structure on top of Hamiltonian learning, but they do so at different levels of abstraction.
Dissipative Hamiltonian Neural Networks (D-HNNs) extend the HNN philosophy by decomposing the vector field into a conservative Hamiltonian part and a dissipative part parameterized, for instance, through a Rayleigh dissipation function; this yields an interpretable separation and enables identification and generalization across friction regimes \cite{SosanyaGreydanus2022DHNN}. Related approaches treat dissipation/forcing as state-dependent external contributions within generalized or pseudo-Hamiltonian parameterizations
\cite{Eidnes2023PHNN}.
A complementary, system-theoretic line is based on energy-based modeling:
port-Hamiltonian neural networks encode dissipation and interconnection through
power ports~\cite{Desai2021PortHamiltonianNN}, while thermodynamics-consistent,
GENERIC/metriplectic, Onsager-principle-based, and identifiable dissipative-learning
frameworks incorporate energy, entropy, dissipation, or reversible-irreversible
structures into the learning process~\cite{GrmelaOettinger1997GENERIC,
Yu2021OnsagerNet,Chen2024CustomThermodynamics,Hernandez2021SPNN,
Zhu2025IdentifiableDissipative}.
Together, these works establish that dissipation can be represented in a structured and interpretable manner, typically by augmenting Hamiltonian structure with additional dissipative mechanisms.

For linearly damped Hamiltonian systems, however, the defining feature is more specific: the exact flow is governed by a \emph{conformal symplectic} law \cite{McLachlanPerlmutter2001}.
This exposes two gaps that are not addressed explicitly by the above learning paradigms.
First, most of these models are formulated in continuous time (as structured vector fields) and rely on a separate time-stepping procedure at rollout; even if the continuous model captures energy balance or dissipation, generic discretization does not guarantee that the resulting one-step map respects the conformal symplectic scaling, which is precisely what controls the correct long-horizon decay factor.
Second, damping identification is often implicit and intertwined with model capacity or regularization, and, in some cases, enforced only approximately via penalty/projection, whereas linear damping induces a globally interpretable one-step dissipation factor  that should be identified directly from snapshot-pair data.
These considerations suggest that, for linearly damped Hamiltonian dynamics, a more aligned learning target is the discrete-time one-step map constrained to the class of conformal symplectic maps, with the net dissipation scalar factor  learned explicitly.

Motivated by this viewpoint, we propose \emph{Conformal Symplectic Networks with damping identification (CSympNet-ID)}. The model learns the discrete-time flow map $\widehat\Phi_{\Theta,h}$ directly from snapshot pairs while enforcing \emph{exact discrete conformal symplecticity}
\[
  (\mathrm{D}\widehat\Phi_{\Theta,h}(z))^\top J\,\mathrm{D}\widehat\Phi_{\Theta,h}(z)=\lambda(h)\,J
\]
for all parameter values, without penalty terms or projection, where $\Theta$ denotes the trainable parameters and $\lambda(h)$ is the learned net dissipation scalar factor.
The architecture ``sandwiches'' an exact symplectic core (e.g., LA-/G-SympNet modules \cite{Jin2020SympNets}) between explicit diagonal scaling layers. 
The scaling factors are parameterized exponentially by a learned scalar damping rate $\hat\gamma$, yielding $\lambda(h)=\exp(-\hat\gamma h)>0$ and making the learned dissipation factor interpretable. 

The main contributions of this work are:
\begin{itemize}
  \item We propose CSympNet-ID, a discrete-time conformal-symplectic map-learning architecture that enforces exact conformal symplecticity by construction, thereby matching the intrinsic geometric dissipation law of linearly damped Hamiltonian systems.

  \item We introduce an interpretable damping-identification mechanism based on  a parsimonious scalar-rate parameterization  that directly identifies the one-step dissipation factor $\lambda(h)=\exp(-\hat{\gamma}h)$ while remaining fully compatible with exact structure preservation.

\item We establish approximation-theoretic guarantees for the proposed framework by proving a scaling-conjugacy factorization for  conformal symplectic maps and deriving  a pointwise-in-step density result for CSympNet-ID.

\item  We provide a systematic empirical validation of CSympNet-ID, evaluating long-horizon accuracy, contraction-law fidelity, damping-rate recovery, and sample efficiency across irregular-step, noisy, nonlinear, and high-dimensional benchmarks.
\end{itemize}

The paper is organized as follows. 
Section~\ref{sec:problem} introduces the problem setting and learning objective.
Section~\ref{sec:conformal} reviews conformal symplectic structure and its discrete analogue.
Section~\ref{sec:method} proposes the CSympNet-ID network.
Section~\ref{sec:theory} proves a pointwise-in-step density result for the proposed model.
Section~\ref{sec:experiments} presents the experimental setup and evaluation protocol for long-horizon prediction and damping identification.
Section~\ref{sec:conclusion} concludes and discusses future directions.

\section{Problem Setting}\label{sec:problem}

Let $z=(q,p)\in\mathbb{R}^{2d}$ and denote the canonical symplectic matrix by
\[
J=\begin{pmatrix}0&I\\-I&0\end{pmatrix},
\]
where $I$ and $0$ denote the $d\times d$ identity and zero matrices, respectively.
We consider the following linearly damped Hamiltonian dynamics with scalar momentum damping,
\begin{equation}\label{eq:damped_system}
\dot q=\nabla_p H(q,p),\qquad
\dot p=-\nabla_q H(q,p)-\gamma p,
\end{equation}
where $H:\mathbb{R}^{2d}\to\mathbb{R}$ is smooth and $\gamma>0$ is a constant damping coefficient. Equivalently,
\[
\dot z = J\nabla H(z) - G z,\qquad
G=\mathrm{diag}(0,\gamma I).
\]

\noindent\textbf{Data.}
We assume discrete observations at sampling times with known positive increments. The training set consists of snapshot triples
\[
\mathcal{D}=\left\{(z_0^{(i)},z_1^{(i)},h^{(i)})\right\}_{i=1}^{N},
\]
where $z_0^{(i)}$ and $z_1^{(i)}$ are the initial and terminal snapshots of the $i$-th pair, respectively, and $h^{(i)}>0$ is the corresponding time increment.
Trajectory data \(\{z_n\}_{n=0}^{N}\) sampled at times \(0=t_0<t_1<\cdots<t_N\), with step sizes \(h_n:=t_n-t_{n-1}\), are included as the special case \((z_0^{(i)},z_1^{(i)},h^{(i)})=(z_{n-1},z_n,h_n)\), \(n=1,\ldots,N\).

\noindent\textbf{Learning objective.}
Given $\mathcal{D}$, our goal is to learn a one-step predictor
$\widehat\Phi_{\Theta,h}:\mathbb{R}^{2d}\to\mathbb{R}^{2d}$
that approximates the  true time-\(h\) flow map \(\Phi_h\) of \eqref{eq:damped_system}.
In addition to accurate prediction, we aim to identify physically meaningful damping information from data.
For linear damping, the intrinsic quantity governing long-time decay is the conformal factor $\lambda^\star(h)$, and we target explicit identification of this factor together with the learned map.

For more general systems
\[
\dot q = \nabla_p H(q,p)-\Gamma_q q,\qquad
\dot p = -\nabla_q H(q,p)-\Gamma_p p,
\]
where $\Gamma_q,\Gamma_p\in\mathbb{R}^{d\times d}$ are constant diagonal matrices, the dynamics can be rewritten as
\[
\dot q=\nabla_p \widetilde H(q,p),
\qquad
\dot p=-\nabla_q \widetilde H(q,p)-\Lambda p,
\qquad
\Lambda:=\Gamma_q+\Gamma_p,
\]
with the modified Hamiltonian
\[
\widetilde H(q,p):=H(q,p)-q^\top\Gamma_q p .
\]
Thus, when $\Lambda=\gamma I$, the system belongs to the scalar linearly damped Hamiltonian class considered in this work, and the induced conformal factor depends only on the total damping rate $\gamma$. More general anisotropic diagonal damping, for which $\Lambda$ is not a scalar multiple of the identity, falls outside the standard scalar-factor conformal symplectic setting and is left for future work.

\section{Conformal Symplecticity}\label{sec:conformal}
We first recall the geometric structure that will be imposed on the learned one-step.
\begin{definition}[Conformal symplectic map]\label{def:csp}
A $C^1$ map $\Phi:\mathbb{R}^{2d}\to\mathbb{R}^{2d}$ is \emph{conformally symplectic}
with scalar factor $\lambda>0$ if
\[
(\mathrm{D}\Phi(z))^\top J\,\mathrm{D}\Phi(z)=\lambda J,\qquad \forall z\in\mathbb{R}^{2d},
\]
where $\mathrm{D}\Phi(z)$ denotes the Jacobian matrix of $\Phi$ with respect to $z$. For a flow family \(\{\Phi_t\}_{t\ge0}\), the corresponding scalar factor may depend on \(t\).
When $\lambda=1$, $\Phi$ is symplectic.
\end{definition}
The next proposition connects this definition to the damped Hamiltonian dynamics introduced in Section~\ref{sec:problem}, showing that
the flow of \eqref{eq:damped_system} is conformally symplectic.
\begin{proposition}[Conformal symplecticity of the damped flow]\label{prop:flow_csp}
Let \(\Phi_t\) be the time-\(t\) flow of \eqref{eq:damped_system}. Then, for every \(t\ge0\),
\begin{equation}\label{eq:flow_factor}
\bigl(\mathrm{D}\Phi_t(z)\bigr)^\top J\,\mathrm{D}\Phi_t(z)=\lambda^\star(t)J,\qquad \lambda^\star(t)=\exp(-\gamma t).
\end{equation}
Thus, for each fixed  time \(t\), the flow map \(\Phi_t\) is conformally symplectic with scalar factor \(\lambda^\star(t)\).
\end{proposition}
\begin{proof}
See \hyperref[prf:proof_flow_csp]{Appendix~\ref*{prf:proof_flow_csp}}.
\end{proof}
To build neural maps with this structure, we use diagonal momentum scaling as the elementary dissipative component.

\begin{lemma}[Scaling is conformally symplectic]\label{lem:scaling_csp}
For all $a>0$, define the diagonal scaling map
\[
\Scale{a}(q,p)=(q, a p).
\]
Then the map $\Scale{a}$ satisfies
\[
(\mathrm{D}\Scale{a})^\top J\,\mathrm{D}\Scale{a} = a\,J.
\]
\end{lemma}
\begin{proof}
See \hyperref[prf:proof_scaling_csp]{Appendix~\ref*{prf:proof_scaling_csp}}.
\end{proof}

The proposed architecture is constructed by composing scaling maps with symplectic modules, so we need the following closure property.

\begin{lemma}[Composition rule]\label{lem:csp_composition}
Let $r\ge 1$, $F:U\to\mathbb{R}^{2d}$ and $G:V\to U$ be $C^r$ maps such that
\[
(\mathrm DF(x))^\top J\,\mathrm DF(x)=\alpha J,\qquad x\in U,
\]
and
\[
(\mathrm DG(y))^\top J\,\mathrm DG(y)=\beta J,\qquad y\in V,
\]
for some constants $\alpha,\beta>0$. Then $F\circ G$ is $C^r$ and satisfies
\[
(\mathrm D(F\circ G)(y))^\top J\,\mathrm D(F\circ G)(y)=\alpha\beta\,J,\qquad y\in V.
\]
\end{lemma}

\begin{proof}
See \hyperref[prf:proof_csp_composition]{Appendix~\ref*{prf:proof_csp_composition}}.
\end{proof}

\section{Method: CSympNet-ID}\label{sec:method}

In this section, we construct a neural one-step map $\widehat\Phi_{\Theta,h}$
for the finite-time transition associated with \eqref{eq:damped_system}.
The map is learned from snapshot pairs and satisfies
\[
\bigl(\mathrm{D}\widehat\Phi_{\Theta,h}(z)\bigr)^\top J\,\mathrm{D}\widehat\Phi_{\Theta,h}(z)=\lambda(h)\,J,\qquad \lambda(h)>0,
\]
for all $z$, where $\lambda(h)$ is the learned scalar conformal factor. We first introduce the symplectic core used in the proposed network.

\noindent\textbf{Symplectic core.}
Let $\Psi_{\theta,h}:\mathbb{R}^{2d}\to\mathbb{R}^{2d}$ be an exact symplectic neural one-step map, such as an LA-/G-SympNet \cite{Jin2020SympNets}, satisfying
\[
(\mathrm{D}\Psi_{\theta,h}(z))^\top J\,\mathrm{D}\Psi_{\theta,h}(z)=J,
\qquad \forall z.
\]
Here $h>0$ denotes the prescribed step size of the autonomous one-step update, and the learnable parameters of the symplectic core are denoted by $\theta$.

\noindent\textbf{CSympNet-ID architecture.}
We compose two identical diagonal scaling maps with the symplectic core:
\begin{equation}\label{eq:csympnet_def}
\widehat\Phi_{\Theta,h}
:=
\Scale{a(h)} \circ \Psi_{\theta,h} \circ \Scale{a(h)},\quad a(h)=\exp\!\left(-\hat\gamma h/2\right),
\end{equation}
where 
$\Theta=(\theta,\hat\gamma)$, with \(\hat\gamma\ge0\).
The corresponding learned conformal factor is 
\begin{equation}\label{eq:learned_lambda}
\lambda(h)=a(h)^2=\exp(-\hat\gamma h).
\end{equation}
This symmetric composition separates conservative transport, modeled by the symplectic one-step core $\Psi_{\theta,h}$, from linear dissipation, captured by the explicit scaling map $\Scale{a(h)}$. 
Architecture of CSympNet-ID is illustrated in Figure \ref{fig:csympnet_arch}.

\begin{figure}[H]
\centering
\resizebox{0.8\textwidth}{!}{
\begin{tikzpicture}[
    font=\small,
    >={Latex[length=2.2mm,width=1.6mm]},
    block/.style={
        draw=blue!60!black,
        rounded corners=5pt,
        thick,
        minimum height=1.25cm,
        minimum width=2.35cm,
        align=center,
        fill=blue!3
    },
    scaleblock/.style={
        draw=blue!60!black,
        rounded corners=5pt,
        thick,
        minimum height=1.25cm,
        minimum width=1.65cm,
        align=center,
        fill=blue!6
    },
    coreblock/.style={
        draw=black!55,
        rounded corners=5pt,
        thick,
        minimum height=1.35cm,
        minimum width=3.55cm,
        align=center,
        fill=gray!5
    },
    module/.style={
        draw=blue!60!black,
        rounded corners=5pt,
        thick,
        minimum height=1.05cm,
        minimum width=2.15cm,
        align=center,
        fill=blue!3
    },
    arrow/.style={->, thick, black!85},
    title/.style={font=\bfseries\small, text=blue!55!black, align=center},
    eqn/.style={align=left, font=\small}
]

% ---------- Panel (a): high-level architecture ----------
\node[font=\bfseries\normalsize] at (-6.75,1.25) {(a)};

\node[block] (input) at (-5.55,0)
    {$z_n$\\[-0.5mm]$=(q_n,p_n)$};

\node[scaleblock] (scalein) at (-3.15,0)
    {$S_{a_n}$};

\node[coreblock] (core) at (0,0)
    {\textbf{exact symplectic core}\\[-0.5mm]$\Psi_{\theta,h_n}$};

\node[scaleblock] (scaleout) at (3.15,0)
    {$S_{a_n}$};

\node[block, minimum width=2.85cm] (output) at (5.65,0)
    {$\widehat z_{n+1}$\\[-0.5mm]$=(\widehat q_{n+1},\widehat p_{n+1})$};

\draw[arrow] (input) -- (scalein);
\draw[arrow] (scalein) -- (core);
\draw[arrow] (core) -- (scaleout);
\draw[arrow] (scaleout) -- (output);

\node[title, above=0.18cm of input] {Input};
\node[title, above=0.18cm of scalein] {Scaling};
\node[title, above=0.18cm of core] {Symplectic core};
\node[title, above=0.18cm of scaleout] {Scaling};
\node[title, above=0.18cm of output] {Output};

% Full CSympNet-ID map and conformal factor: placed on the left side
\node[eqn, anchor=west] (fullEq) at (-6.35,-1.65)
{
$\widehat\Phi_{\Theta,h_n}
=
S_{a_n}\circ \Psi_{\theta,h_n}\circ S_{a_n}$\\[1mm]
$ a_n:=a(h_n)=\exp(-\hat\gamma h_n/2)$
};

% ---------- Panel (b): expanded symplectic core ----------
\node[module] (m1) at (-3.55,-4.35)
    {symplectic\\module $1$};

\node[module] (m2) at (-0.95,-4.35)
    {symplectic\\module $2$};

\node[align=center] (dots) at (1.20,-4.35)
    {$\cdots$};

\node[module] (mM) at (3.55,-4.35)
    {symplectic\\module $M$};

\node[
    draw=black!45,
    rounded corners=7pt,
    thick,
    fit=(m1)(m2)(dots)(mM),
    inner xsep=0.75cm,
    inner ysep=0.90cm,
    label={[font=\small]above:$\Psi_{\theta,h_n}$}
] (corebox) {};

\node[font=\bfseries\normalsize] at ($(corebox.north west)+(0.38,-0.38)$) {(b)};

\draw[arrow] ($(corebox.west)+(-0.95,0)$) -- (m1.west);
\draw[arrow] (m1) -- (m2);
\draw[arrow] (m2) -- (dots);
\draw[arrow] (dots) -- (mM);
\draw[arrow] (mM.east) -- ($(corebox.east)+(0.95,0)$);

\node[font=\itshape\small] at ($(corebox.south)+(0,0.28)$)
    {e.g., LA-/G-SympNet stack};

% Single vertical dashed guide line: only indicates expansion of the symplectic core
\draw[dashed, thick, black!50]
    (core.south) -- (corebox.north);

\end{tikzpicture}
}
\caption{
Architecture of CSympNet-ID.
For a prescribed step size $h_n$, the input state $z_n=(q_n,p_n)$ is first mapped by the diagonal momentum scaling
$S_{a_n}(q,p)=(q,a_n p)$.
The scaled state is then passed through the exact symplectic core $\Psi_{\theta,h_n}$, and the same scaling is applied once more to obtain the predicted state
$\widehat z_{n+1}=(\widehat q_{n+1},\widehat p_{n+1})$.
The complete one-step map is
$\widehat\Phi_{\Theta,h_n}=S_{a_n}\circ\Psi_{\theta,h_n}\circ S_{a_n}$,
where
$a_n=a(h_n)=\exp(-\hat\gamma h_n/2)$ and
$\lambda(h_n)=a_n^2=\exp(-\hat\gamma h_n)$.
The lower panel expands only the symplectic core as a stack of exact symplectic modules, such as LA-/G-SympNet blocks.
}
\label{fig:csympnet_arch}
\end{figure}
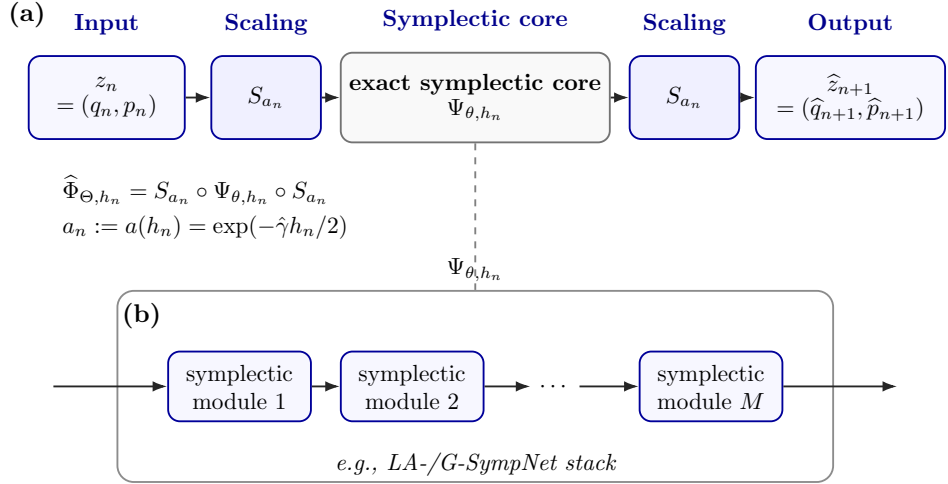

The symmetric placement of the two damping layers is motivated by the splitting viewpoint used in conformal symplectic integrators for linearly damped Hamiltonian systems \cite{BhattFloydMoore2016}. Such dynamics can be viewed as the combination of a conservative Hamiltonian flow and a linear dissipative flow, and a natural time-centered discretization applies half damping, then a symplectic update, and then the remaining half damping. This leads to the sandwich form in \eqref{eq:csympnet_def}. 
A single scaling layer before or after the symplectic core would also preserve conformal symplecticity if its factor were chosen correctly, but it would apply all dissipation on one side of the conservative transport. Interleaving damping factors inside the SympNet core is also possible, but independent layer-wise damping rates would be non-identifiable because only their product determines the global conformal factor. Thus, the symmetric sandwich gives a minimal, time-centered, and interpretable parameterization tied to one scalar damping rate.

\noindent\textbf{Training objective and scalar-rate parameterization.}
For snapshot pairs, we minimize the one-step mean-squared error
\begin{equation}\label{eq:one_step_loss}
\mathcal L_{\rm 1step}(\Theta)
=
\frac{1}{N}\sum_{i=1}^N
\left\|
z^{(i)}_1-\widehat\Phi_{\Theta,h^{(i)}}(z^{(i)}_0)
\right\|_2^2 .
\end{equation}
In the implementation, the nonnegativity constraint is enforced by setting
\[
\hat\gamma=|\beta|,
\qquad \beta\in\mathbb R,
\]
where \(\beta\) is an unconstrained trainable scalar.
This loss is used both for training and for model selection on the validation set. 
The scaling factor is parameterized exponentially as in \eqref{eq:learned_lambda}; hence, the constraint \(\hat\gamma\ge0\) implies \(\lambda(h)\le1\) for all \(h>0\).

\noindent\textbf{Variable step sizes.}
For variable-step data with increments $\{h^{(i)}\}$, CSympNet-ID reuses the same model parameters $\Theta=(\theta,\hat\gamma)$ across all snapshot pairs. 
The prescribed step size enters the model in two ways. 
First, following the step-dependent SympNet construction for irregularly sampled Hamiltonian data \cite{Jin2020SympNets}, $h^{(i)}$ scales the increments of the triangular SympNet shear modules, so that the resulting core map $\Psi_{\theta,h^{(i)}}$ remains symplectic for every prescribed $h^{(i)}>0$.
Second, \(h^{(i)}\) sets the explicit dissipative factors \(a(h^{(i)})\) and \(\lambda(h^{(i)})\) defined similarly as in \eqref{eq:csympnet_def} and \eqref{eq:learned_lambda}.
During rollout with step sizes $\{h_n\}$, CSympNet-ID applies
$\widehat z_{n+1}=\widehat\Phi_{\Theta,h_n}(\widehat z_n)$, thereby enforcing a consistent conformal contraction law across mixed sampling intervals.

\begin{proposition}[Exact conformal symplecticity of CSympNet-ID]\label{prop:exact_csp}
For any \(\Theta=(\theta,\hat\gamma)\) with \(\hat\gamma\ge0\), and any \(h>0\), map \eqref{eq:csympnet_def} satisfies
\begin{equation}\label{eq:exact_csp}
\bigl(\mathrm{D}\widehat\Phi_{\Theta,h}(z)\bigr)^\top J\,\mathrm{D}\widehat\Phi_{\Theta,h}(z)
=
\lambda(h)\,J,
\qquad
\lambda(h)=\exp(-\hat\gamma h),
\end{equation}
for all $z\in\mathbb{R}^{2d}$.
\end{proposition}
\begin{proof}
See \hyperref[prf:proof_exact_csp]{Appendix~\ref*{prf:proof_exact_csp}}.
\end{proof}

\section{Approximation Theory for Conformal Symplectic Maps}\label{sec:theory}
This section provides an approximation-theoretic analysis of CSympNet-ID in the  conformal symplectic class. The key step is a scaling-conjugacy argument, which reduces conformal symplectic approximation to symplectic approximation.

\noindent\textbf{$C^r$ norms and density on compact sets.}
Let $U\subset\mathbb{R}^m$ be open and $W\subset U$ be compact. For $f\in C^r(W;\mathbb{R}^n)$, define
\[
\|f\|_{C^r(W)}:=\sum_{|\alpha|\le r}\max_{1\le i\le n}\ \sup_{x\in W}\bigl|D^\alpha f_i(x)\bigr|.
\]
A family $\mathcal{F}\subset C^r(U;\mathbb{R}^n)$ is said to be \emph{$r$-uniformly dense on compacta} in a set $\mathcal{G}\subset C^r(U;\mathbb{R}^n)$
if for any $g\in\mathcal{G}$, any compact $W\subset U$, and any $\varepsilon>0$, there exists $f\in\mathcal{F}$ such that
\[
\|f-g\|_{C^r(W)}<\varepsilon.
\]

We next introduce the map classes used in the approximation result.

\noindent\textbf{Symplectic and conformal symplectic map classes.}
Let $\mathrm{SP}^r(U)$ denote the set of $C^r$ maps on $U\subset\mathbb{R}^{2d}$ satisfying
\[
(\mathrm{D}\Psi(z))^\top J\,\mathrm{D}\Psi(z)=J,\qquad \forall z\in U.
\]
For a given scalar factor $\lambda>0$, let $\mathrm{CSP}^r_\lambda(U)$ denote the set of $C^r$ maps on $U$ satisfying
\[
(\mathrm{D}\Phi(z))^\top J\,\mathrm{D}\Phi(z)=\lambda J,\qquad \forall z\in U.
\]

For the approximation analysis below, we fix a step size \(h>0\) and a scalar conformal factor \(\lambda(h)>0\). 
In CSympNet-ID, this factor is given by the exponential parameterization in \eqref{eq:learned_lambda} once \(\hat\gamma\) is fixed.
The next theorem records the scaling-conjugacy relation underlying the density result.

\begin{theorem}[Scaling-conjugacy factorization at a fixed  step]\label{thm:conjugacy}
Let \(h>0\) be fixed, and set
$a:=\sqrt{\lambda(h)}$. If $\Phi_h\in \mathrm{CSP}^r_{\lambda(h)}(U)$, then
\[
\Psi_h:=\Scale{a^{-1}}\circ \Phi_h\circ \Scale{a^{-1}}
\]
is well defined on \(\Scale{a}(U)\) and belongs to \(\mathrm{SP}^r(\Scale{a}(U))\). Equivalently, on \(U\), we have
\[
\Phi_h=\Scale{a}\circ\Psi_h\circ\Scale{a}.
\]
\end{theorem}
\begin{proof}
See \hyperref[prf:proof_conjugacy]{Appendix~\ref*{prf:proof_conjugacy}}.
\end{proof}

The factorization above reduces the density transfer to the approximation capacity of the symplectic core. 
We use the SympNet approximation theorem of Jin et al.~\cite{Jin2020SympNets}: LA-/G-SympNets with an \(r\)-finite activation are \(r\)-uniformly dense on compacta in \(\mathrm{SP}^r(V)\) for every open set \(V\subset\mathbb R^{2d}\). 
Here \(r\)-finite means \( \sigma\in C^r(\mathbb R)\) and
\[
0<\int_{-\infty}^{\infty}\left|\sigma^{(r)}(x)\right|\,\mathrm{d}x<\infty,
\]
which is satisfied by sigmoid-type activations; for G-SympNet cores, the scalar generating functions can be chosen as antiderivatives of the sigmoid activation. 
For each prescribed \(h>0\), we assume that the fixed-\(h\) core family \(\{\Psi_{\theta,h}\}\) satisfies this density property, so the following result is pointwise in \(h\).

\begin{theorem}[Pointwise-in-step density of CSympNet-ID]\label{thm:density_transfer}
Let \(r\ge1\), \(h>0\), and let \(U\subset\mathbb R^{2d}\) be open. Set
$a:=\sqrt{\lambda(h)}$.
Assume that, for this prescribed step size \(h\), the neural core family
$\{\Psi_{\theta,h}\}$
is \(r\)-uniformly dense on compacta in \(\mathrm{SP}^r(V)\) for every open set \(V\subset\mathbb{R}^{2d}\). 
Then the CSympNet-ID family
$\left\{
\Scale{a}\circ\Psi_{\theta,h}\circ\Scale{a}
\right\}$ is \(r\)-uniformly dense on compacta in
\(\mathrm{CSP}^r_{\lambda(h)}(U)\).
\end{theorem}
\begin{proof}
See \hyperref[prf:proof_density_transfer]{Appendix~\ref*{prf:proof_density_transfer}}.
\end{proof}

\section{Experimental Setup and Results}\label{sec:experiments}

This section reports the numerical evaluation of CSympNet-ID on linearly damped Hamiltonian benchmarks. Section~\ref{sec:exp_systems} presents the test problems and data-generation procedure. Section~\ref{sec:models_metrics} describes the compared models, training protocol, and evaluation metrics. Section~\ref{sec:numerical_experiments} presents the three groups of experiments, and Section~\ref{sec:exp_discussion} summarizes the main observations.

\subsection{Test problems and data generation}\label{sec:exp_systems}

For each test problem, we specify the Hamiltonian $H(q,p)$, and the corresponding damped equations of motion are obtained from \eqref{eq:damped_system} with damping coefficient $\gamma=0.05$. The variables satisfy $q,p\in\mathbb{R}^{d}$, so the phase-space dimension is $2d$. In the first test problem S1, we use $d=1$, while in the second S2 and third S3 experiments we use $d=10$. We further consider $d=100$ extensions for S2 and S3 in the high-dimensional tests. Unless otherwise stated, rollouts are evaluated on $[0,10]$.

\paragraph{S1: Damped harmonic oscillator.}
We consider $d=1$, and the Hamiltonian is
\[
H(q,p)=\frac12\omega^2q^2+\frac12p^2,
\]
where $\omega$ is the natural frequency. In all reported experiments, we set $\omega=1$.
This two-dimensional benchmark tests damping-rate estimation under irregular training increments $h\in[0.05,0.15]$.

\paragraph{S2: Damped spring-mass chain.}
The main experiment uses $d=10$, and the high-dimensional extension uses $d=100$. The Hamiltonian is
\[
H(q,p)=\frac12\|p\|^2+\frac{k}{2}\sum_{i=1}^{d-1}(q_{i+1}-q_i)^2,
\]
where $k$ is the nearest-neighbor spring stiffness, set to $k=1$ in all experiments. This linear chain tests whether the scalar conformal factor can be recovered accurately as the phase-space dimension increases.

\paragraph{S3: Damped nonlinear cubic oscillator.}
The main experiment uses $d=10$, and the high-dimensional extension uses $d=100$. The nonlinear benchmark uses
\[
H(q,p)=\frac12\omega^2\|q\|^2+\frac12\|p\|^2+
\alpha\left(\sum_{i=1}^{d}{q_i}\right)^3,
\]
where $\omega$ controls the linear oscillation frequency and $\alpha$ controls the strength of the cubic nonlinearity. We set $\omega=1$. The main $d=10$ experiment uses $\alpha=0.005$, while the $d=100$ extension uses $\alpha=10^{-4}$ to keep the reference dynamics stable over $[0,10]$.

\paragraph{Data generation.}
The clean snapshot data are generated by a second-order symmetric conformal-symplectic splitting with a St\"ormer-Verlet  core. 
The St\"ormer-Verlet step is
\[
 p_{n+1/2}=p_n-\frac{h}{2}\nabla V(q_n),\qquad
 q_{n+1}=q_n+h p_{n+1/2},\qquad
 p_{n+1}=p_{n+1/2}-\frac{h}{2}\nabla V(q_{n+1}).
\]
The full reference update is
\[
z_{n+1}=\Scale{a(h_{\mathrm{ref}})}\circ \Psi^{\mathrm{ref}}_{h_{\mathrm{ref}}}\circ \Scale{a(h_{\mathrm{ref}})}(z_n),\qquad
 a(h_{\mathrm{ref}})=\exp(-\gamma h_{\mathrm{ref}}/2),
\]
with $h_{\mathrm{ref}}=10^{-3}$. Thus the reference update is  conformally symplectic with the scalar target factor $\lambda^\star(h_{\mathrm{ref}})=\exp(-\gamma h_{\mathrm{ref}})$. To check that the reported conclusions are not artifacts of this structure-preserving reference integrator, we also repeated representative tests using exact matrix-exponential solutions for the linear S1 and S2 systems and a high-accuracy adaptive Runge-Kutta solver with tight tolerances for the nonlinear S3 system; the qualitative conclusions on rollout accuracy and damping-rate recovery remained unchanged, as discussed in Section~\ref{sec:exp_discussion}.

In the noise-robustness experiments, observations are perturbed by component-wise relative
Gaussian noise. For each snapshot pair, we set
\[
\widetilde z^{(i)}_\ell
=
z^{(i)}_\ell+\sigma s_\ell\odot \xi^{(i)}_\ell,
\qquad
\xi^{(i)}_\ell\sim \mathcal N(0,I_{2d}),\quad \ell\in\{0,1\},
\]
where \(i\) indexes the snapshot pair, \(\ell=0\) and \(\ell=1\) denote the initial and terminal
snapshots of the pair, respectively, and \(\odot\) denotes component-wise multiplication.
\(s_\ell\in\mathbb R^{2d}\) is the component-wise empirical standard deviation of the clean
snapshots \(\{z^{(m)}_\ell\}_{m=1}^{1000}\). The scalar \(\sigma\) controls the relative noise
level, and we use \(\sigma\in\{0,10^{-4},10^{-3},10^{-2}\}\).

\noindent\textbf{Train-validation-test split.}
The training, validation, and test sets are generated independently. The training set consists of one-step snapshot pairs $(z_0,z_1,h)$, where $z_0$ is sampled as an initial condition at time $t=0$, $h>0$ is the prescribed time increment, and $z_1$ is the corresponding clean terminal state generated by the reference flow approximation over time $h$. The validation set has the same structure and is generated from the same distribution, with size $N_{\mathrm{val}}=N_{\mathrm{train}}/4$. The test set consists of held-out rollout trajectories sampled at times $t_n$ on $[0,10]$, where $z_n$ denotes the corresponding reference state at time $t_n$ for $n=0,\ldots,N_t$ and $t_n=n h_{\mathrm{test}}$ in the fixed-step tests. Each learned model is initialized from the same held-out initial conditions and then iterated with the prescribed test step sizes to evaluate long-horizon rollout accuracy.

\subsection{Benchmark models and evaluation metrics}\label{sec:models_metrics}

We compare CSympNet-ID with the following baselines:
\begin{itemize}
\item \textbf{SympNet}: the symplectic neural-network architecture introduced by Jin et al.~\cite{Jin2020SympNets}, used here as an exact symplectic map-learning reference model. In each benchmark, the SympNet baseline uses the same LA-SympNet or G-SympNet core architecture and comparable hyperparameters as the symplectic core \(\Psi_{\theta,h}\) in CSympNet-ID, but without the conformal scaling layers.
\item \textbf{MapBaseline}: an unconstrained multi-layer perceptron (MLP) that directly learns the one-step map without imposing any geometric structure.
\item \textbf{ODEBaseline}: a neural vector-field model in the Neural ODE framework~\cite{Chen2018NeuralODE}.
\item \textbf{DHNN} and \textbf{DHNN2}: dissipative Hamiltonian neural-network baselines based on the formulation in~\cite{SosanyaGreydanus2022DHNN}. DHNN2 denotes a restricted variant with an isotropic linear dissipative term.
\item \textbf{pHNN} and \textbf{pHNN2}: port-Hamiltonian neural-network baselines based on the formulation in~\cite{Desai2021PortHamiltonianNN}. pHNN2 denotes a restricted variant with an isotropic linear dissipative term.
\end{itemize}

The compared models are grouped into map-based and vector-field models. 
CSympNet-ID, SympNet, and MapBaseline directly approximate finite-time flow maps, whereas ODEBaseline, DHNN, DHNN2, pHNN, and pHNN2 learn vector fields that are numerically integrated at rollout. 
Map-based models are trained with the one-step loss in \eqref{eq:one_step_loss}, while vector-field models are trained with the midpoint loss
\begin{equation}\label{eq:midpoint_loss}
L_{\mathrm{midpoint}}(\theta)
=\frac{1}{N}\sum_{i=1}^{N}
\|
\frac{z_1^{(i)}-z_0^{(i)}}{h^{(i)}}
-\widehat f_{\theta}
\left(
\frac{z_1^{(i)}+z_0^{(i)}}{2}
\right)
\|_2^2.
\end{equation}
All models are trained on the same snapshot pairs and evaluated on the same rollout tasks. The benchmark protocol, training settings, and evaluation details are summarized in Table~\ref{tab:experiment_protocol}. 
For consistency across vector-field baselines, all vector-field models, including DHNN2 and pHNN2, are evaluated using the adaptive dopri5 solver with the tolerances reported in Table~\ref{tab:experiment_protocol}. We also tested conformal-symplectic rollout integrators for the restricted dissipative variants when applicable. 
These rollout evaluations were substantially more expensive and did not change the qualitative conclusions. 
Therefore, we report the dopri5 results in the main comparisons.

\begin{table}[H]
\centering
\caption{Benchmark protocol and reproducibility settings used in the reported experiments.}
\label{tab:experiment_protocol}
\small
\begin{tabularx}{\linewidth}{@{}p{0.07\linewidth}p{0.25\linewidth}p{0.08\linewidth}p{0.11\linewidth}X@{}}
\toprule
\multicolumn{5}{@{}l}{\textbf{Benchmark protocol}} \\
\midrule
System & Damped system & $d$ & Train pairs & Train/test step protocol \\
\midrule
S1 & harmonic oscillator & $1$ & $80$ 
& $h_{\mathrm{train}}\in[0.05,0.15]$, $h_{\mathrm{test}}=0.2$ \\
S2 & spring-mass chain & $10/100$ & $400$ 
& fixed $h_{\mathrm{train}}=0.1$, $h_{\mathrm{test}}=0.2$ \\
S3 & damped nonlinear cubic oscillator & $10/100$ & $400$ 
& fixed $h_{\mathrm{train}}=0.1$, $h_{\mathrm{test}}=0.2$ \\
\midrule
\multicolumn{5}{@{}l}{\textbf{Reproducibility settings}} \\
\midrule
\multicolumn{2}{@{}l}{Item} & \multicolumn{3}{l@{}}{Setting} \\
\midrule
\multicolumn{2}{@{}l}{Damping coefficient} 
& \multicolumn{3}{p{0.6\linewidth}@{}}{$\gamma=0.05$ for all benchmarks} \\
\multicolumn{2}{@{}l}{Reference solver} 
& \multicolumn{3}{p{0.6\linewidth}@{}}{conformal-symplectic splitting method  with St\"ormer-Verlet core} \\
\multicolumn{2}{@{}l}{Reference step} 
& \multicolumn{3}{p{0.6\linewidth}@{}}{$h_{\mathrm{ref}}=10^{-3}$} \\
\multicolumn{2}{@{}l}{Rollout interval} 
& \multicolumn{3}{p{0.6\linewidth}@{}}{$[0,10]$} \\
\multicolumn{2}{@{}l}{Observation noise} 
& \multicolumn{3}{p{0.6\linewidth}@{}}{$\sigma\in\{0,10^{-4},10^{-3},10^{-2}\}$} \\
\multicolumn{2}{@{}l}{Optimizer} 
& \multicolumn{3}{p{0.6\linewidth}@{}}{AdamW, full-batch training} \\
\multicolumn{2}{@{}l}{Initial learning rate / weight decay} 
& \multicolumn{3}{p{0.6\linewidth}@{}}{$10^{-3}$ / $5\times10^{-4}$} \\
\multicolumn{2}{@{}l}{Checkpoint selection} 
& \multicolumn{3}{p{0.6\linewidth}@{}}{lowest validation loss within a fixed epoch budget} \\
\multicolumn{2}{@{}l}{Loss function} 
& \multicolumn{3}{p{0.6\linewidth}@{}}{one-step loss  \eqref{eq:one_step_loss} for map-based models and midpoint loss \eqref{eq:midpoint_loss} for vector-field models} \\
\multicolumn{2}{@{}l}{Vector-field baseline solver} 
& \multicolumn{3}{p{0.6\linewidth}@{}}{adaptive dopri5, rtol $10^{-7}$, atol $10^{-9}$} \\
\multicolumn{2}{@{}l}{Evaluation data} 
& \multicolumn{3}{p{0.6\linewidth}@{}}{all reported results are evaluated on test trajectories generated with seed \(0\) over \(T\in[0,10]\)} \\
\multicolumn{2}{@{}l}{ Tables / trajectory figures} 
& \multicolumn{3}{p{0.6\linewidth}@{}}{models trained with a fixed representative seed $42$} \\
\multicolumn{2}{@{}l}{Bar plots} 
& \multicolumn{3}{p{0.6\linewidth}@{}}{models trained with seeds $\{10,20,30,40,50,60,70,80,90,100\}$; bars show the median and error bars show the interquartile range over the 10 seeds} \\
\bottomrule
\end{tabularx}
\end{table}

For \(M\) rollout trajectories sampled at times \(0=t_0<t_1<\cdots<t_{N_t}\) on \([0,10]\), 
with \(h_n:=t_n-t_{n-1}\), the reported trajectory metric is the long-horizon quadrature \(L^2\) rollout error
\[
\mathcal{E}_{L^2}
=\left(
\sum_{m=1}^{M}
\sum_{n=1}^{N_t}
\|
z^{(m)}(t_n)-\widehat z^{(m)}(t_n)\|_2^2
h_n
\right)^{1/2}.
\]
This corresponds to a right-rectangle quadrature approximation of the time-integrated squared rollout error.

We also evaluate all methods using the same conformal-symplectic diagnostics. For a learned time-$h$ map $\widehat\Phi_h$, the conformal-law residual with respect to the target factor $\lambda^\star(h)=\exp(-\gamma h)$ is defined as
\begin{equation}\label{eq:residual_metric}
R(z)=\|
\bigl(\D\widehat\Phi_h(z)\bigr)^\top
J
\D\widehat\Phi_h(z)-\lambda^\star(h)J
\|_F .
\end{equation}
We report both $\mathbb{E}[R]$ and $\max R$ over the test points. To obtain a scalar diagnostic conformal factor for any learned map, we use the Frobenius projection
\begin{equation}\label{eq:lambda_projection_metric}
\widehat\lambda(z)
=
\frac{
\left\langle
\bigl(\D\widehat\Phi_h(z)\bigr)^\top
J\D\widehat\Phi_h(z),
J
\right\rangle_F
}{
\|J\|_F^2
}.
\end{equation}
The corresponding conformal-factor residual is reported as
$\mathbb{E}\bigl[|\widehat\lambda-\lambda^\star(h)|\bigr]$.

Equations~\eqref{eq:residual_metric} and \eqref{eq:lambda_projection_metric} are applied uniformly to CSympNet-ID and to all baselines. For CSympNet-ID, the learned damping rate \(\hat\gamma\) gives the explicit structural factor \(\lambda(h)=\exp(-\hat\gamma h)\), and the projected factor in \eqref{eq:lambda_projection_metric} agrees with it up to machine accuracy. For baselines that do not enforce a scalar conformal factor, $\widehat\lambda(z)$ is used only as a post-hoc diagnostic and may vary with $z$. For vector-field baselines, the learned vector field is first integrated over the same step size $h$ to obtain the induced time-$h$ map $\widehat\Phi_h$, after which the same Jacobian-based diagnostics are applied. In all Jacobian-based diagnostics, $D\widehat{\Phi}_h(z)$ is computed by automatic differentiation in PyTorch. The residuals and rollout errors are compared only within the same benchmark.

\subsection{Model performance}\label{sec:numerical_experiments}

\subsubsection{Noise-free long-horizon prediction and structure diagnostics}\label{sec:exp_I}

Table~\ref{tab:l2err} reports the long-horizon
quadrature $L^2$ errors from a representative single run with training seed 42. Multi-seed variability is evaluated separately in the noise-robustness, sample-efficiency and held-out-step sweep experiments. CSympNet-ID gives the smallest error among the compared models on all three main benchmarks while using nearly the fewest parameters among the tested models.  Figure~\ref{fig:trajectory_main} shows representative rollouts, including the phase portrait of the first coordinate pair, the first-coordinate trajectory, accumulated rollout error, and energy evolution.
In these rollouts, CSympNet-ID tracks the reference phase portraits and first-coordinate trajectories most closely across S1-S3 and keeps the accumulated rollout error smallest throughout the simulation window.
SympNet shows systematic deviations because it enforces the conservative factor $\lambda=1$, while MapBaseline can drift away from the correct phase-space region despite its larger parameter count.
The ODEBaseline, DHNN/DHNN2, and pHNN/pHNN2 baselines generally improve over the unstructured map baseline and often reproduce the qualitative dissipative trend, but their accumulated errors remain larger than those of CSympNet-ID.
The energy panels show the same pattern: the dissipative Hamiltonian and port-Hamiltonian baselines capture energy decay more reasonably than SympNet, yet CSympNet-ID most closely matches the reference decay because it enforces the target discrete conformal contraction law directly at the map level.
\begin{table}[H]
\centering
\caption{Long-horizon
quadrature $L^2$ error at $\sigma=0$.}
\label{tab:l2err}
\small
\begin{tabular}{lrrrrrr}
\toprule
\multirow{2}{*}{Model} & \multicolumn{2}{c}{S1} & \multicolumn{2}{c}{S2, $d=10$} & \multicolumn{2}{c}{S3, $d=10$} \\
\cmidrule(lr){2-3}\cmidrule(lr){4-5}\cmidrule(lr){6-7}
& Params & $L^2$ error & Params & $L^2$ error & Params & $L^2$ error \\
\midrule
CSympNet-ID(ours) & 115  & \textbf{2.175e-04} & 4111  & \textbf{2.513e-02} & 4201  & \textbf{4.096e-02} \\
SympNet           & 114  &         2.108e-01  & 4110  &         1.192e+00  & 4200  &         1.245e+00  \\
MapBaseline       & 1250 &         1.611e-01  & 6868  &         2.274e+01  & 6868  &         3.472e+01  \\
ODEBaseline       & 1218 &         6.321e-03  & 6804  &         2.605e-01  & 6804  &         1.764e-01  \\
DHNN              & 2374 &         5.891e-03  & 11538 &         4.664e-01  & 11538 &         7.159e-01  \\
DHNN2             & 1190 &         8.686e-03  & 5970  &         2.453e-01  & 5970  &         3.973e-01  \\
pHNN              & 1193 &         9.777e-03  & 6369  &         1.791e-01  & 6369  &         3.699e-01  \\
pHNN2             & 1190 &         5.248e-03  & 5970  &         2.305e-01  & 5970  &         4.992e-01  \\
\bottomrule
\end{tabular}
\end{table}

\begin{table}[H]
\centering
\caption{Target conformal-law residuals and factor mismatch at $\sigma=0$.}
\label{tab:diagnostics}
\scriptsize
\resizebox{\linewidth}{!}{
\begin{tabular}{lrrrrrrrrr}
\toprule
\multirow{2}{*}{Model} & \multicolumn{3}{c}{$\mathbb{E}[R]$} & \multicolumn{3}{c}{$\max R$} & \multicolumn{3}{c}{$\mathbb{E}[|\hat\lambda-\lambda^\star|]$} \\
\cmidrule(lr){2-4}\cmidrule(lr){5-7}\cmidrule(lr){8-10}
& S1 & S2 & S3 & S1 & S2 & S3 & S1 & S2 & S3 \\
\midrule
CSympNet-ID(ours) & \textbf{4.6e-07} & \textbf{2.7e-06} & \textbf{2.6e-05} & \textbf{8.4e-07} & \textbf{3.0e-06} & \textbf{2.6e-05} & \textbf{3.9e-07} & \textbf{5.7e-07} & \textbf{5.8e-06} \\
SympNet           &         1.4e-02  &         4.4e-02  &         4.4e-02  &         1.4e-02  &         4.4e-02  &         4.4e-02  &         1.0e-02  &         1.0e-02  &         1.0e-02  \\
MapBaseline       &         1.1e-02  &         3.6e-01  &         3.6e-01  &         1.3e-02  &         3.8e-01  &         3.7e-01  &         7.7e-03  &         7.8e-02  &         7.8e-02  \\
ODEBaseline       &         6.2e-05  &         1.2e-02  &         1.1e-02  &         1.8e-04  &         1.6e-02  &         1.4e-02  &         4.4e-05  &         1.8e-04  &         1.7e-04  \\
DHNN              &         1.1e-04  &         2.3e-02  &         2.4e-02  &         1.8e-04  &         3.3e-02  &         3.0e-02  &         7.8e-05  &         6.0e-04  &         1.3e-03  \\
DHNN2             &         2.1e-05  &         1.1e-02  &         1.7e-02  &         4.4e-05  &         1.3e-02  &         2.3e-02  &         1.5e-05  &         3.9e-04  &         7.5e-04  \\
pHNN              &         6.1e-05  &         3.3e-03  &         2.3e-03  &         8.0e-05  &         3.4e-03  &         2.5e-03  &         4.3e-05  &         2.2e-04  &         3.8e-05  \\
pHNN2             &         5.8e-05  &         2.0e-03  &         2.6e-03  &         8.2e-05  &         2.3e-03  &         2.8e-03  &         4.1e-05  &         3.9e-04  &         5.2e-04  \\
\bottomrule
\end{tabular}}
\end{table}

\begin{figure}[H]
\centering
\begin{subfigure}[t]{0.99\linewidth}
  \centering
  \includegraphics[width=\linewidth]{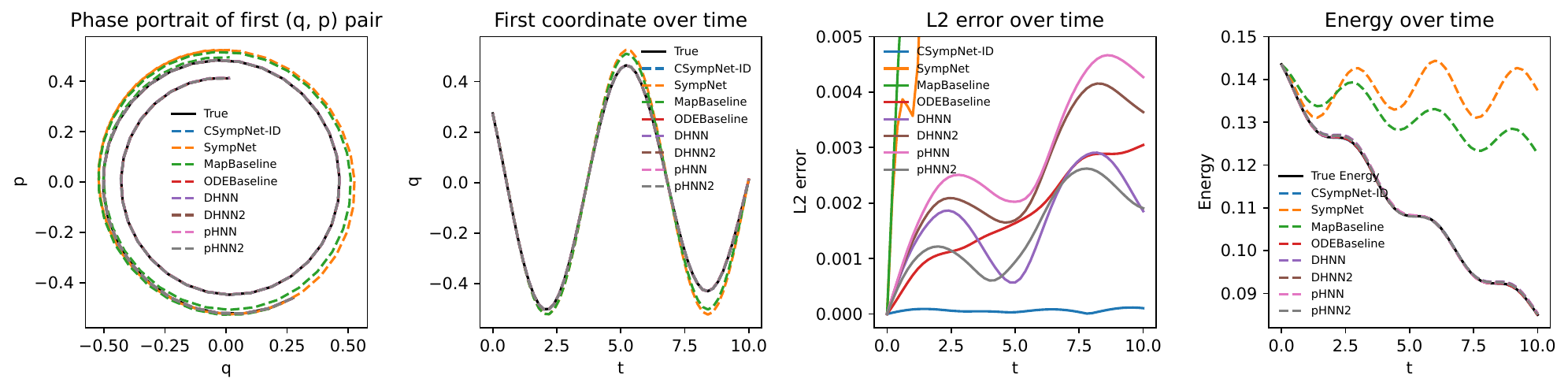}
  \caption{S1: damped harmonic oscillator.}
\end{subfigure}

\begin{subfigure}[t]{0.99\linewidth}
  \centering
  \includegraphics[width=\linewidth]{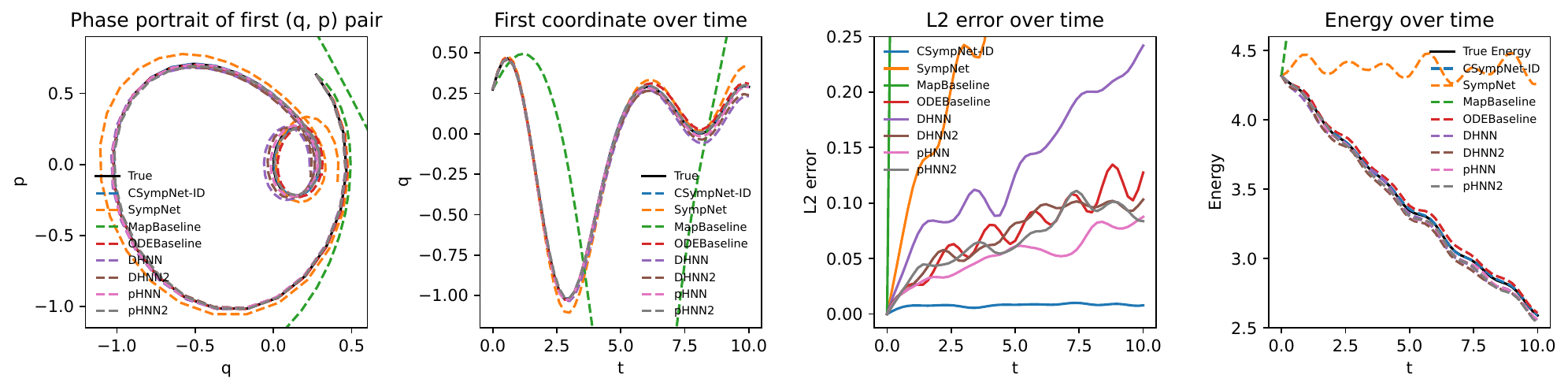}
  \caption{S2: damped spring-mass chain ($d=10$).}
\end{subfigure}

\begin{subfigure}[t]{0.99\linewidth}
  \centering
  \includegraphics[width=\linewidth]{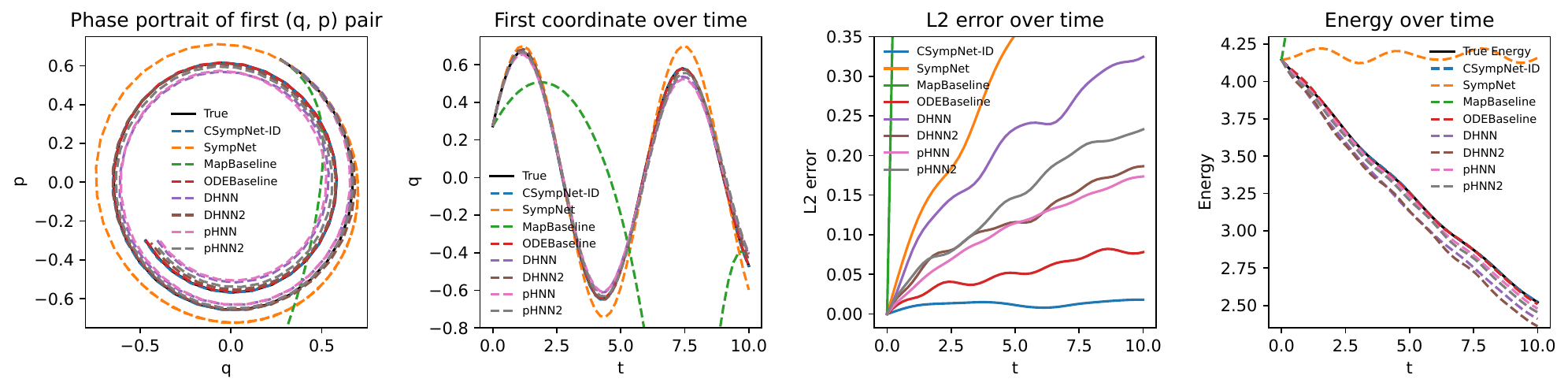}
  \caption{S3: damped nonlinear cubic oscillator ($d=10$).}
\end{subfigure}
\caption{Representative long-horizon rollouts for the S1--S3 benchmarks at $\sigma=0$. Each panel shows the phase portrait of the first coordinate pair, the first coordinate over time, the accumulated rollout error, and the energy evolution.}
\label{fig:trajectory_main}
\end{figure}

\begin{table}[H]
\centering
\caption{Damping-rate estimation by CSympNet-ID at $\sigma=0$.}
\label{tab:gamma_id}
\small
\begin{tabular}{lrrr}
\toprule
System & $\gamma$ & $\hat\gamma$ & $|\gamma-\hat\gamma|$ \\
\midrule
S1         & 5.0000e-02 & 4.9998e-02 & 1.7911e-06 \\
S2, $d=10$ & 5.0000e-02 & 5.0003e-02 & 3.1821e-06 \\
S3, $d=10$ & 5.0000e-02 & 4.9971e-02 & 2.9355e-05 \\
\bottomrule
\end{tabular}
\end{table}

The structure diagnostics in Tables~\ref{tab:diagnostics} and~\ref{tab:gamma_id} show that the improvement is not only a trajectory-fitting effect. 
CSympNet-ID gives the smallest target conformal-law residuals and factor mismatches across all three benchmarks, with factor errors at the \(10^{-7}\)--\(10^{-6}\) level and damping-rate errors below \(3.0\times10^{-5}\). 
By contrast, SympNet has an irreducible factor mismatch of order \(10^{-2}\) because it enforces \(\lambda=1\), and MapBaseline does not reliably recover the target contraction law despite being an unconstrained map learner.
The ODEBaseline, DHNN/DHNN2, and pHNN/pHNN2 baselines reduce the mismatch relative to SympNet and MapBaseline in several cases, but their target-law residuals and factor errors remain consistently larger than those of CSympNet-ID.

\subsubsection{Robustness under noisy observations}\label{sec:exp_noise}

Figures~\ref{fig:error_noise_all} and~\ref{fig:lambda_noise_all} summarize the noise-robustness results over $10$ random seeds. These seeds are $\{10,20,30,40,50,60,70,80,90,100\}$. 
Figure~\ref{fig:error_noise_all} shows that the advantage of CSympNet-ID becomes more pronounced on the more challenging benchmarks S2 and S3 than on the low-dimensional oscillator S1. 
As the noise level increases, CSympNet-ID exhibits a relatively mild growth of rollout error, whereas MapBaseline and several generic baselines degrade more rapidly, especially on the higher-dimensional spring-mass chain and the nonlinear cubic benchmark. 
The structured dissipative baselines, including ODEBaseline, DHNN/DHNN2, and pHNN/pHNN2, are generally more stable than the unstructured map baseline, but they still accumulate larger errors than CSympNet-ID in most noisy cases.

\begin{figure}[H]
\centering
\begin{subfigure}[t]{0.65\linewidth}
  \centering
  \includegraphics[width=\linewidth]{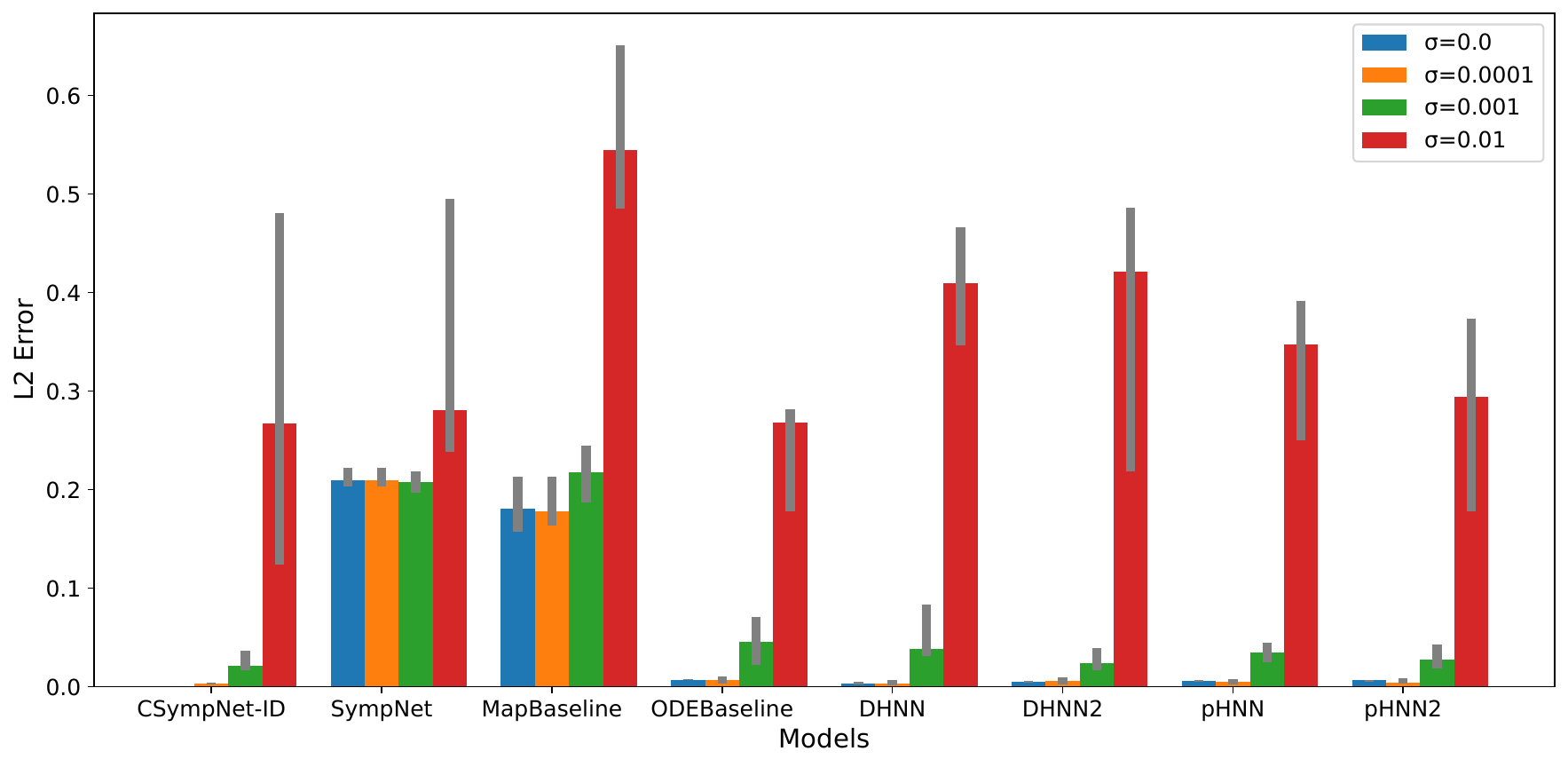}
  \caption{S1: damped harmonic oscillator.}
\end{subfigure}
\begin{subfigure}[t]{0.65\linewidth}
  \centering
  \includegraphics[width=\linewidth]{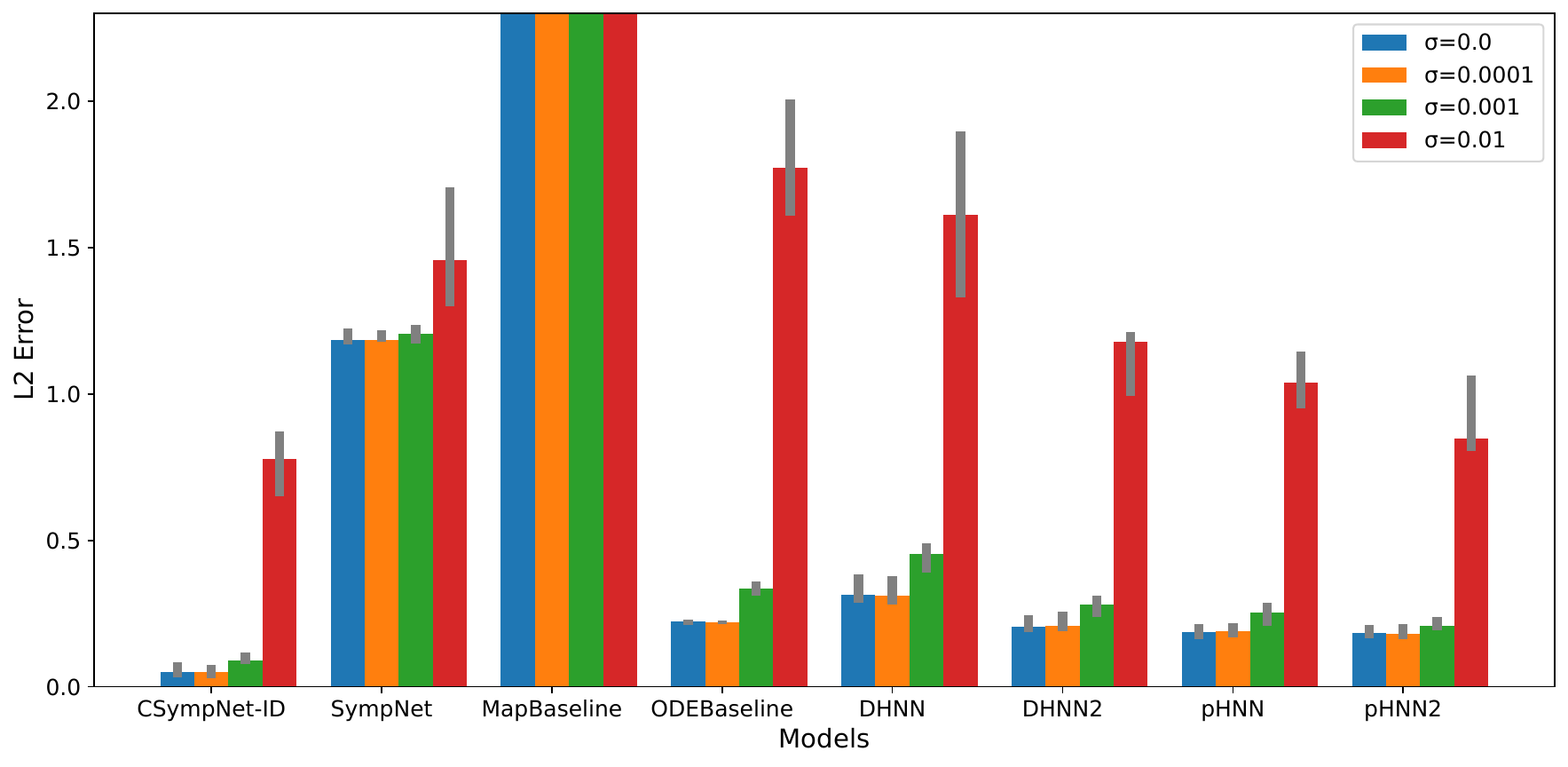}
  \caption{S2: damped spring-mass chain.}
\end{subfigure}
\begin{subfigure}[t]{0.65\linewidth}
  \centering  \includegraphics[width=\linewidth]{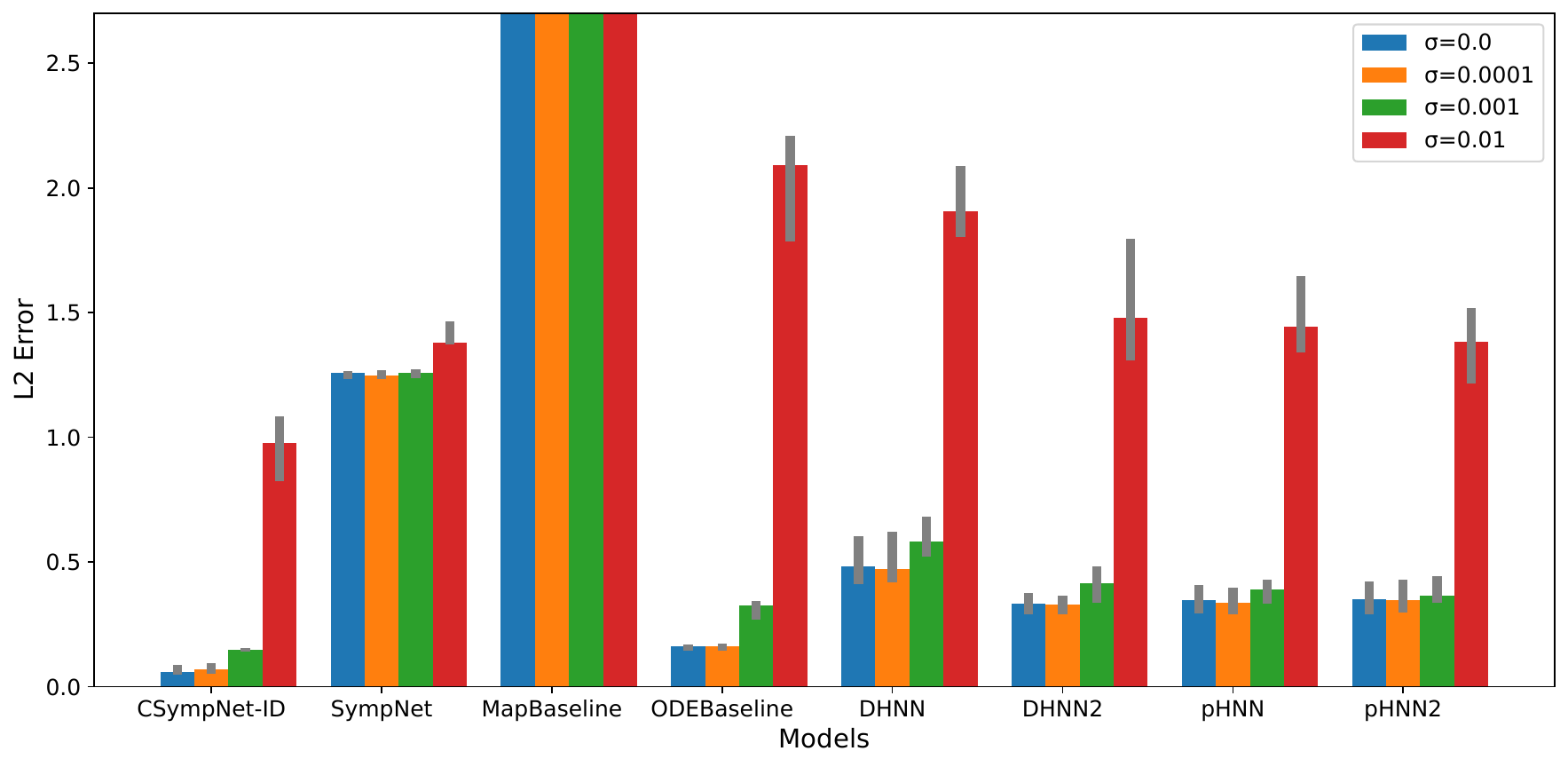}
  \caption{S3: damped nonlinear cubic oscillator.}
\end{subfigure}
\caption{Long-horizon rollout $L^2$ error under increasing observation noise. Bars show the median over 10 random seeds, and error bars show the interquartile range. For readability, MapBaseline bars/error bars exceeding the plotted range are clipped. This clipping affects only the display, not the reported statistics.}
\label{fig:error_noise_all}
\end{figure}

Figure~\ref{fig:lambda_noise_all} shows a complementary structural trend. CSympNet-ID stays close to the target contraction factor $\lambda^\star$ in most tested settings and exhibits relatively small variation of the learned factor as the noise level increases. One exception occurs for S3 at the largest noise level, where the median projected factor of ODEBaseline is closer to $\lambda^\star$. 
Nevertheless, across the three benchmarks, CSympNet-ID remains among the most stable methods in recovering the discrete contraction factor. SympNet is constrained to $\lambda=1$ and therefore cannot represent the damped contraction, while the other dissipative baselines may recover the factor reasonably in some cases but can show larger variance or bias under stronger noise. These results suggest that the proposed conformal-symplectic map structure improves the stability of the learned discrete contraction law in most tested cases, although very noisy observations can reduce or remove this advantage for individual benchmarks.

\begin{figure}[H]
\centering
\begin{subfigure}[t]{0.65\linewidth}
  \centering
  \includegraphics[width=\linewidth]{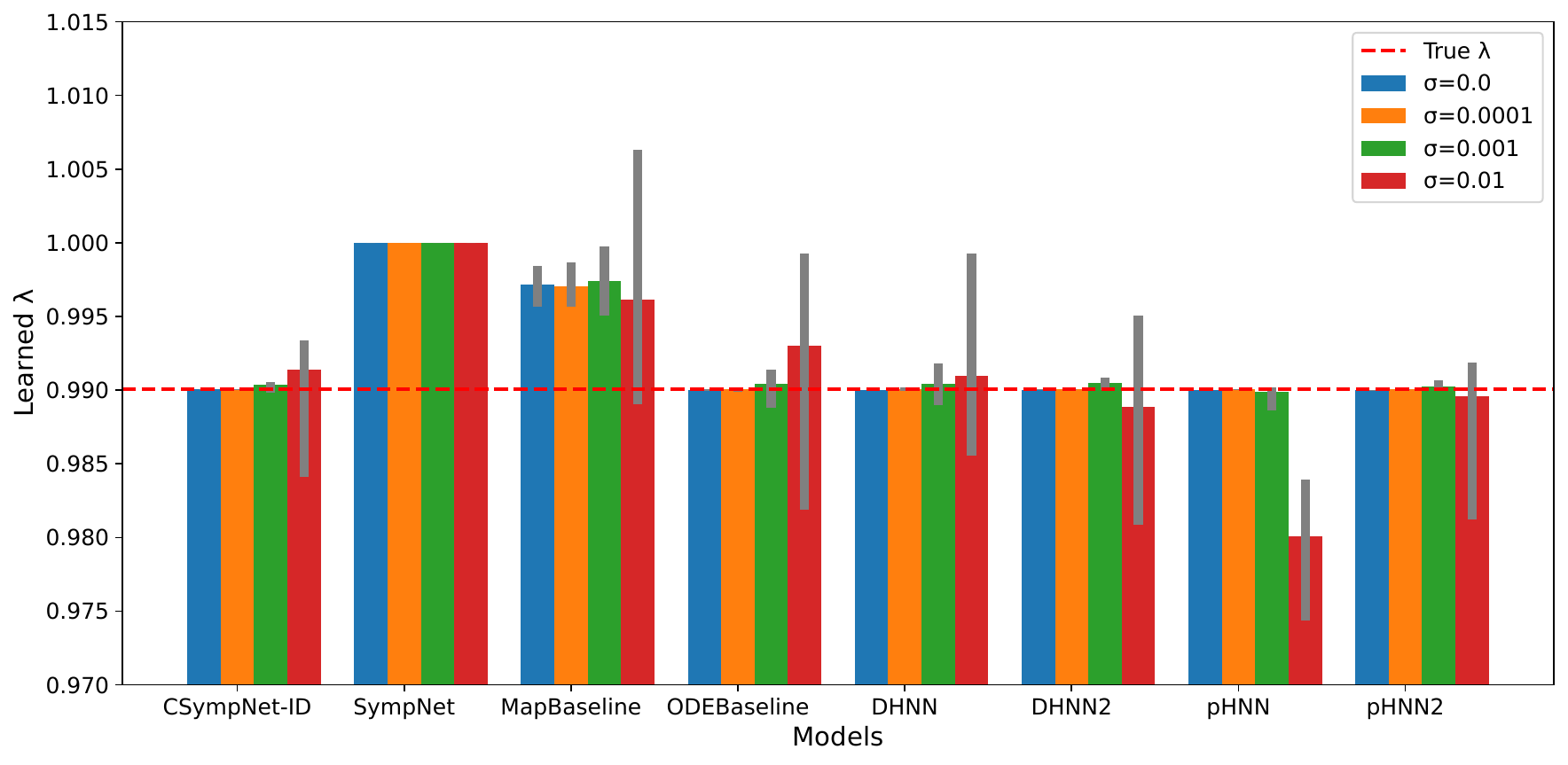}
  \caption{S1: damped harmonic oscillator.}
\end{subfigure}
\begin{subfigure}[t]{0.65\linewidth}
  \centering
  \includegraphics[width=\linewidth]{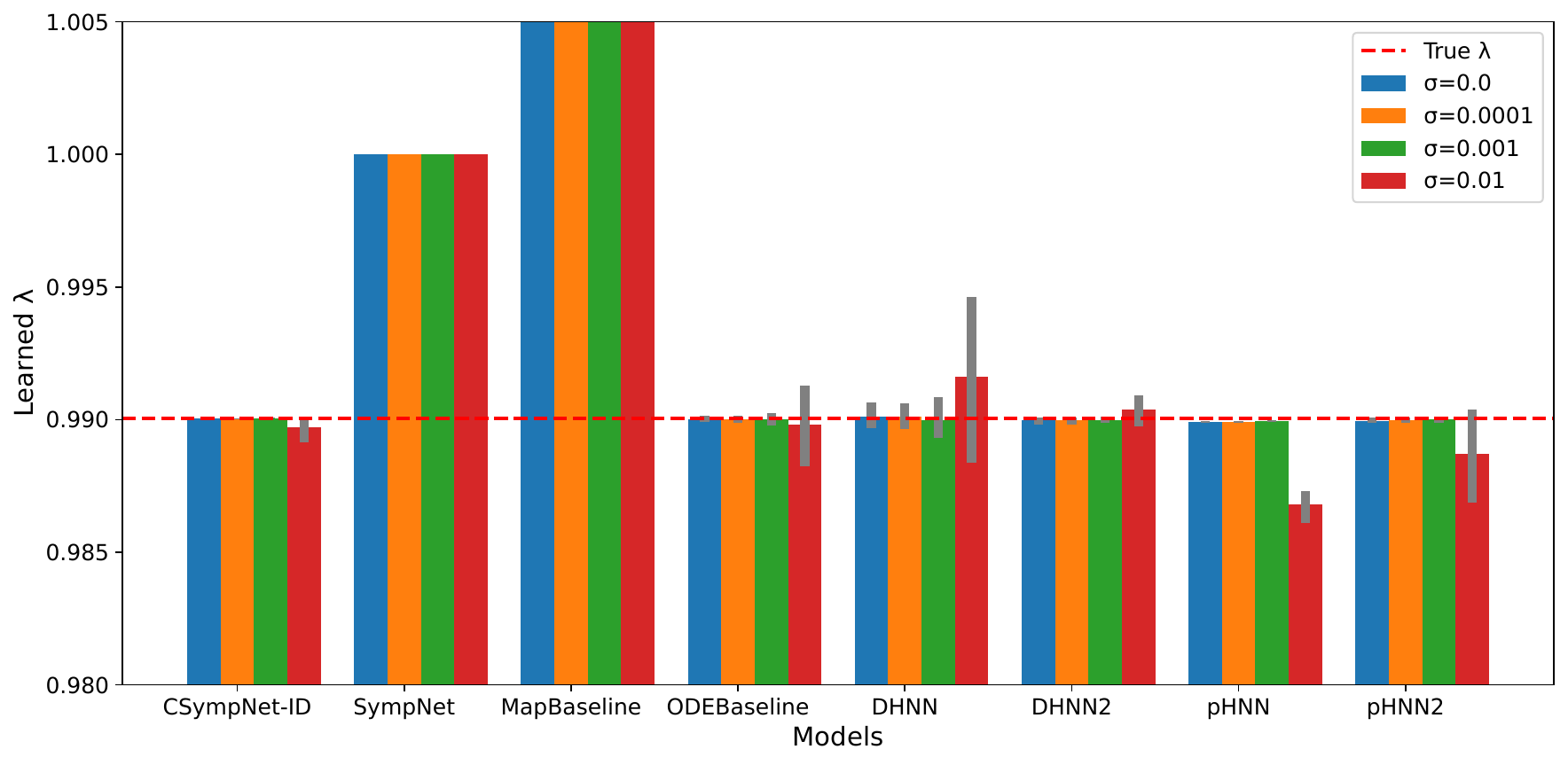}
\caption{S2: damped spring-mass chain.}
\end{subfigure}
\begin{subfigure}[t]{0.65\linewidth}
  \centering
  \includegraphics[width=\linewidth]{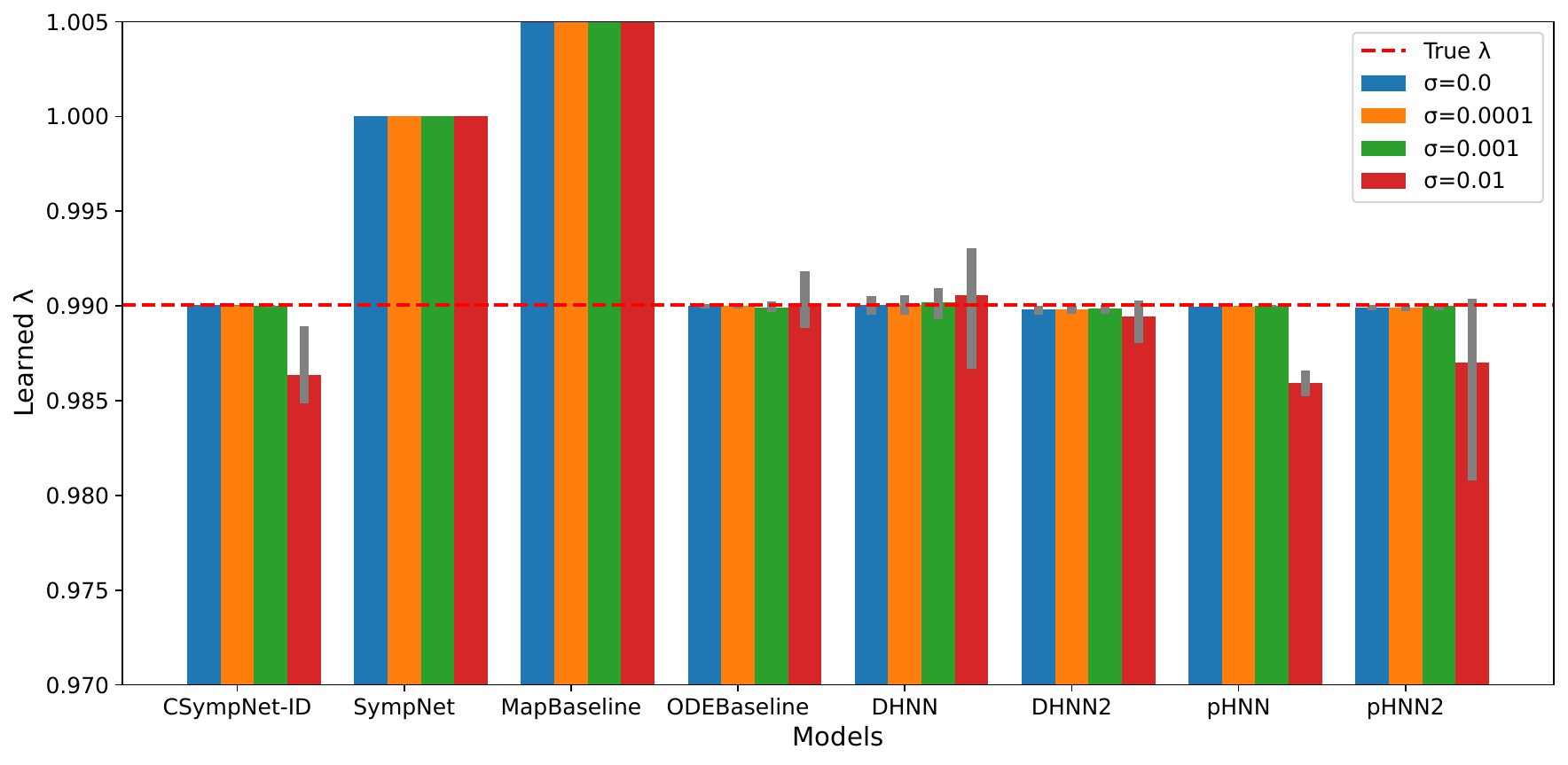}
  \caption{S3: damped nonlinear cubic oscillator.}
\end{subfigure}
\caption{Learned conformal factor under increasing observation noise. The dashed horizontal line marks the true factor $\lambda^\star$. Bars show the median over $10$ random seeds, and error bars show the interquartile range. For readability, the vertical range is restricted around \(\lambda^\star\), and MapBaseline bars/error bars outside this range are clipped.}
\label{fig:lambda_noise_all}
\end{figure}

\subsubsection{High-dimensional extension}\label{sec:exp_highdim}

We further test S2 and S3 with $d=100$, so that the corresponding phase-space dimension is $2d=200$. Table~\ref{tab:highdim_results} reports high-dimensional results from a fixed representative run with training seed~$42$, evaluated on test trajectories generated with seed~$0$. Figure~\ref{fig:highdim_traj} shows representative rollouts. Within each \(d=100\) benchmark, CSympNet-ID gives the smallest long-horizon
quadrature $L^2$ error, target conformal-law residual, and factor mismatch among the compared models. 
MapBaseline and ODEBaseline exhibit clear long-horizon degradation in this regime. Although SympNet gives a visually reasonable projected trajectory in the S3 representative
rollout, its fixed choice \(\lambda=1\) leads to a systematic contraction-factor mismatch.
The dissipative Hamiltonian and port-Hamiltonian baselines behave more favorably than the
generic baselines but still yield larger errors and less accurate contraction-factor recovery
than CSympNet-ID.

\begin{table}[H]
\centering
\caption{High-dimensional extension results for S2 and S3 with $d=100$.}
\label{tab:highdim_results}
\scriptsize
\resizebox{\linewidth}{!}{
\begin{tabular}{lrrrrrrrr}
\toprule
\multirow{2}{*}{Model} & \multicolumn{4}{c}{S2, $d=100$} & \multicolumn{4}{c}{S3, $d=100$} \\
\cmidrule(lr){2-5}\cmidrule(lr){6-9}
& Params & $L^2$ error & $\mathbb{E}[R]$ & $\mathbb{E}[|\hat\lambda-\lambda^\star|]$ & Params & $L^2$ error & $\mathbb{E}[R]$ & $\mathbb{E}[|\hat\lambda-\lambda^\star|]$ \\
\midrule
CSympNet-ID(ours) & 161101 & \textbf{2.609e+00} & \textbf{2.9e-05} & \textbf{2.0e-06} & 161101 & \textbf{2.632e+00} & \textbf{3.8e-04} & \textbf{2.7e-05} \\
SympNet           & 161100 &         7.067e+00  &         1.4e-01  &         1.0e-02  & 161100 &         4.642e+00  &         1.4e-01  &         1.0e-02  \\
MapBaseline       & 168904 &         2.391e+09  &         1.2e+01  &         3.9e-01  & 168904 &         1.336e+09  &         9.9e+00  &         3.7e-01  \\
ODEBaseline       & 168648 &         6.150e+01  &         2.2e+00  &         5.1e-03  & 168648 &         4.460e+01  &         1.4e+00  &         2.7e-03  \\
DHNN              & 275010 &         3.773e+01  &         8.2e-01  &         3.9e-03  & 275010 &         1.345e+01  &         4.1e-01  &         1.1e-03  \\
DHNN2             & 157506 &         2.293e+01  &         1.0e+00  &         2.4e-02  & 157506 &         1.440e+01  &         9.5e-01  &         2.9e-02  \\
pHNN              & 197505 &         1.102e+01  &         2.0e-01  &         4.8e-03  & 197505 &         1.017e+01  &         1.1e-01  &         2.1e-03  \\
pHNN2             & 157506 &         1.808e+01  &         2.0e-01  &         1.3e-02  & 157506 &         3.172e+01  &         8.3e-01  &         4.8e-02  \\
\bottomrule
\end{tabular}}
\end{table}

\begin{figure}[H]
\centering
\begin{subfigure}[t]{0.95\linewidth}
  \centering
  \includegraphics[width=\linewidth]{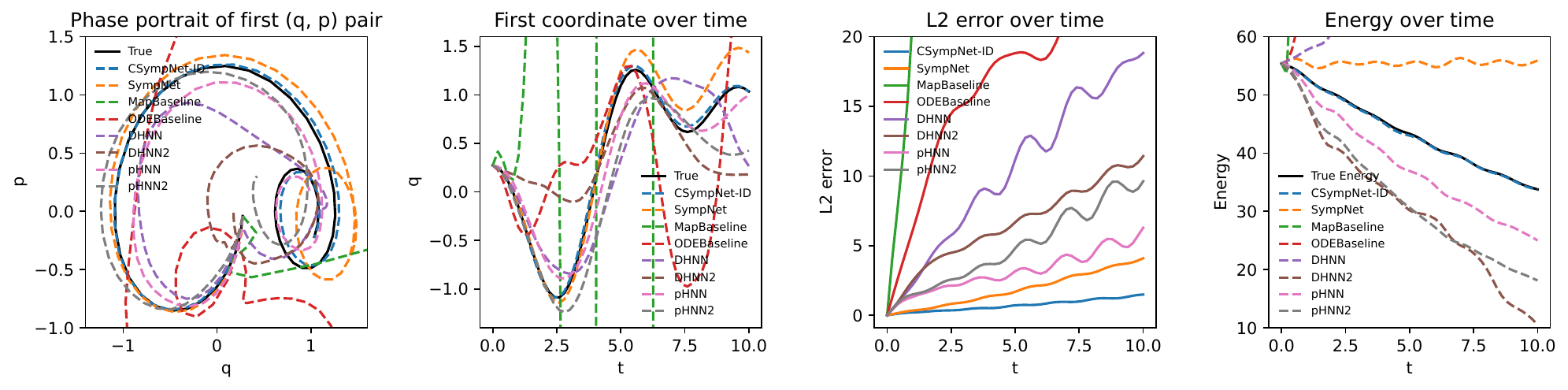}
  \caption{S2 high-dimensional extension: damped spring-mass chain ($d=100$).}
\end{subfigure}

\begin{subfigure}[t]{0.95\linewidth}
  \centering
  \includegraphics[width=\linewidth]{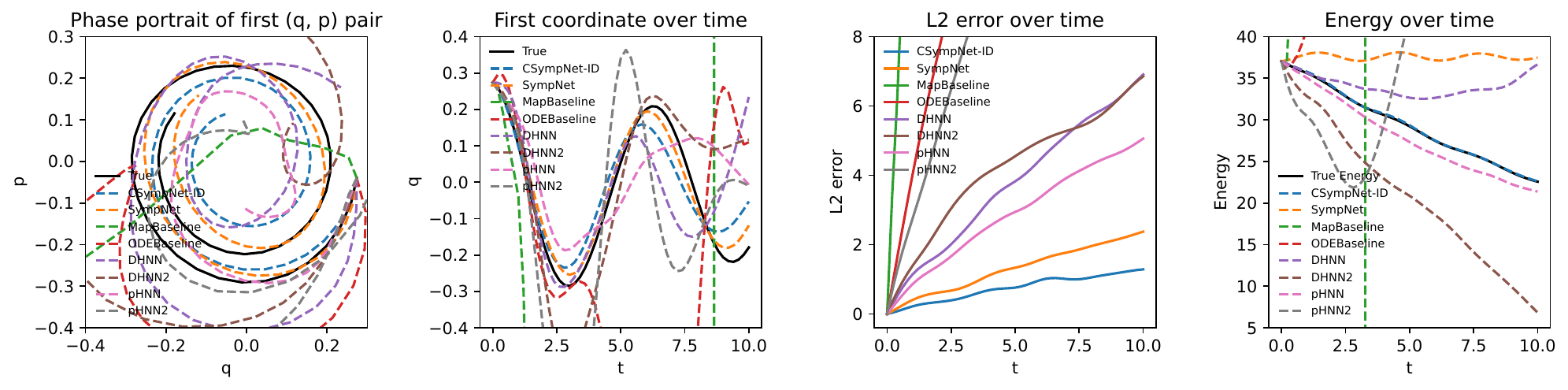}
  \caption{S3 high-dimensional extension: damped nonlinear cubic oscillator ($d=100$).}
\end{subfigure}
\caption{Representative long-horizon  rollouts for the high-dimensional benchmarks with $\sigma=0$.}
\label{fig:highdim_traj}
\end{figure}

\subsection{Discussion}\label{sec:exp_discussion}

The experiments support the main design principle of CSympNet-ID: for scalar linearly damped Hamiltonian systems, the learned discrete-time map should reproduce the conformal contraction law at the map level, rather than only encode a generic Hamiltonian, symplectic, or dissipative bias. 
The comparison with SympNet isolates this point most directly. 
SympNet uses essentially the same symplectic-network mechanism as the proposed model, but fixes the conformal factor to $\lambda=1$. 
It therefore preserves oriented area instead of reproducing the damping-induced contraction, which leads to systematic errors in long-horizon rollouts, energy decay, and contraction-factor diagnostics.

The comparison with MapBaseline shows the limitation of unconstrained one-step map learning. 
Although an unrestricted map can fit short-time snapshot pairs in simple low-dimensional settings, it does not encode the target dissipative geometry. 
This becomes particularly problematic in the higher-dimensional S2 and S3 benchmarks, where MapBaseline exhibits unstable rollouts and large target conformal-law residuals. 
These results indicate that accurate local fitting alone is not sufficient for recovering the geometric contraction law governing the damped dynamics.

The ODEBaseline, DHNN/DHNN2, and pHNN/pHNN2 comparisons show a more subtle distinction. 
These models incorporate continuous-time vector-field structure, and the dissipative Hamiltonian and port-Hamiltonian variants are generally more stable than the unrestricted map baseline. 
However, they still learn a continuous-time model that must be converted into a discrete one-step map through numerical integration at rollout. 
As a result, the learned discrete map is not constrained to satisfy the exact target conformal-symplectic law. 
This distinction is important for the restricted variants DHNN2 and pHNN2: they include an isotropic linear dissipative prior, which is closer to the true scalar damping mechanism than a fully generic dissipative component, but their target-law residuals, factor mismatches, and damping-rate recovery still remain less accurate than those of CSympNet-ID in the reported tests. 
Thus, imposing a compatible continuous-time damping prior improves over unstructured learning, but does not replace enforcing the correct discrete conformal contraction factor directly. For the restricted dissipative variants DHNN2 and pHNN2, we also tested a conformal-symplectic rollout integrator on the noise-free S1--S3 benchmarks. 
It improved their conformal-factor diagnostics but still remained less accurate than CSympNet-ID, while the long-horizon $L^2$ errors changed little and the computational cost was substantially higher; therefore, we report the adaptive dopri5 results in the main comparisons.

The high-dimensional experiments further emphasize this difference. 
In the \(d=100\) tests, generic baselines exhibit substantial long-horizon degradation, whereas the dissipative Hamiltonian and port-Hamiltonian baselines show less severe rollout degradation but still produce larger long-horizon quadrature errors and less accurate contraction-factor recovery than CSympNet-ID. 
Among the baselines other than CSympNet-ID, SympNet gives relatively reasonable rollout behavior in the S3 representative trajectory, but its fixed choice \(\lambda=1\) prevents recovery of the damped conformal factor. 
Repeating the noise-free \(d=100\) tests over 10 training seeds gives the same trend in rollout error and conformal-factor recovery, further supporting the overall advantage of CSympNet-ID in this setting. 
Because the reported \(L^2\) error and target-law residuals are unnormalized raw quantities, these comparisons are interpreted within each fixed benchmark dimension.

The noise-robustness experiments in Section~\ref{sec:exp_noise}, together with sample-efficiency and held-out-step experiments in Appendices~\ref{app:sample_efficiency}--\ref{app:heldout_step}, are consistent with the same conclusion. 
They suggest that the explicit conformal-symplectic parameterization improves rollout stability under limited data, noisy observations, and shifted test step sizes. 
Under stronger observation noise, some vector-field dissipative baselines can become closer in rollout error, but CSympNet-ID remains more consistent  in the structure diagnostics among the compared models and exhibits less variation in the learned damping factor as reported in  Figure~\ref{fig:lambda_noise_all}. We further repeated representative tests with alternative reference solutions to rule out artifacts from the second-order conformal-symplectic splitting used as the main data generator. We used exact matrix-exponential solutions for the linear S1 and S2 systems and a high-accuracy adaptive Runge-Kutta solver with tight tolerances for the nonlinear S3 system, and the main qualitative conclusions on rollout accuracy and damping-rate recovery remained unchanged.

Overall, the numerical evidence indicates that exact discrete conformal symplecticity is a useful inductive bias when the dominant dissipative mechanism is a scalar linear damping rate. 
The results should not be read as a claim for arbitrary anisotropic or state-dependent dissipation; rather, they show that in the scalar-factor setting studied here, explicitly learning and enforcing the one-step conformal factor yields a favorable combination of long-horizon accuracy, contraction-law fidelity, scalar-rate estimation, and rollout stability.

\FloatBarrier

\section{Conclusion}\label{sec:conclusion}

We have proposed CSympNet-ID, a conformal-symplectic map-learning framework for scalar linearly damped Hamiltonian systems. 
The method learns a discrete one-step map directly from snapshot pairs while enforcing the exact  conformal-symplectic law by construction. 
By sandwiching an exact symplectic neural core between explicit diagonal scaling layers, CSympNet-ID preserves discrete conformal symplecticity for all parameter values and provides an interpretable scalar-rate estimate of the damping factor.

The theoretical analysis shows that conformal-symplectic maps can be reduced to symplectic maps through a simple scaling conjugacy. This yields a pointwise-in-step density result for the proposed architecture by transferring existing SympNet approximation guarantees to the conformal-symplectic setting.

The numerical experiments provide consistent empirical support for this design. On the irregular-step damped oscillator S1, the damped spring-mass chain S2, and the damped nonlinear cubic oscillator S3, CSympNet-ID achieves the smallest rollout errors among the compared models in the representative noise-free runs. 
The structure diagnostics further indicate that these gains are accompanied by more accurate recovery of the target conformal contraction law and scalar damping rate in the representative runs. The noise-robustness, sample-efficiency, held-out-step, and high-dimensional extension experiments further support the benefit of enforcing the discrete conformal contraction law, especially for long-horizon prediction and stable damping-factor recovery.

Overall, the results support exact discrete conformal symplecticity as a useful inductive bias for learning scalar linearly damped Hamiltonian dynamics from snapshot data.  Future work will consider state-dependent conformal factors, anisotropic dissipative extensions, non-canonical Poisson generalizations, and operator-learning extensions for PDE-scale dissipative dynamics.

\section*{Data and code availability}
The code will be made public after the paper is published.

 \section*{Declaration of competing interest}
The authors declare that they have no known competing financial interests or personal relationships 
that could have appeared to influence the work reported in this paper.

\section*{Acknowledgements}
This work is supported by the National Natural Science Foundation of China (Grant Nos. 12501580, 92470119, 12301677).

\appendix
\section{Proofs and Diagnostics}\label{sec:appendix}

\subsection{Proof of Proposition~\ref{prop:flow_csp}}\label{prf:proof_flow_csp}
\begin{proof}
\hypertarget{prf:proof_flow_csp}{}
Let $z(t)=\Phi_t(z_0)$ and set $M(t):=\mathrm{D}\Phi_t(z_0)$. The variational equation along the flow reads
\[
\dot M(t)=A(t)M(t),\qquad A(t)=J\nabla^2 H(z(t))-G,
\]
where $G=\diag(0,\gamma I)$. Differentiating $M(t)^\top J M(t)$ gives
\[
\frac{d}{dt}\bigl(M(t)^\top J M(t)\bigr)=M(t)^\top\bigl(A(t)^\top J+J A(t)\bigr)M(t).
\]
Since $\nabla^2H$ is symmetric, one has
\[
\bigl(J\nabla^2 H\bigr)^\top J+J\bigl(J\nabla^2 H\bigr)=0.
\]
Moreover,
\[
(-G)^\top J+J(-G)=-\gamma J.
\]
Therefore,
\[
\frac{d}{dt}\bigl(M(t)^\top J M(t)\bigr)=-\gamma\,M(t)^\top J M(t).
\]
Using $M(0)=I$, we obtain
\[
M(t)^\top J M(t)=e^{-\gamma t}J,
\]
which proves \eqref{eq:flow_factor}.
\end{proof}

\subsection{Proof of Lemma~\ref{lem:scaling_csp}}\label{prf:proof_scaling_csp}
\begin{proof}
\hypertarget{prf:proof_scaling_csp}{}
Using the notation $\Scale{a}(q,p)=(q,ap)$, the map $\Scale{a}$ is linear and its Jacobian is constant:
\[
\mathrm{D}\Scale{a}=
\begin{pmatrix}
I & 0\\
0 & a I
\end{pmatrix}.
\]
Let $S:=\mathrm{D}\Scale{a}$. We compute $S^\top J S$ explicitly by
\[
S^\top J S=
\begin{pmatrix}
I & 0\\
0 & a I
\end{pmatrix}
\begin{pmatrix}
0&I\\
-I&0
\end{pmatrix}
\begin{pmatrix}
I & 0\\
0 & a I
\end{pmatrix}
=a\,J.
\]

\end{proof}

\subsection{Proof of Lemma~\ref{lem:csp_composition}}\label{prf:proof_csp_composition}
\begin{proof}
\hypertarget{prf:proof_csp_composition}{}
By the chain rule,
\[
\mathrm{D}(F\circ G)(y)=\mathrm{D}F(G(y))\,\mathrm{D}G(y).
\]
Hence
\begin{align*}
(\mathrm{D}(F\circ G)(y))^\top J\,\mathrm{D}(F\circ G)(y)
&=(\mathrm{D}G(y))^\top (\mathrm{D}F(G(y)))^\top J\,\mathrm{D}F(G(y))\,\mathrm{D}G(y)\\
&=(\mathrm{D}G(y))^\top (\alpha J)\,\mathrm{D}G(y)
 =\alpha\,(\mathrm{D}G(y))^\top J\,\mathrm{D}G(y)\\
&=\alpha\beta\,J.
\end{align*}
\end{proof}

\subsection{Proof of Proposition~\ref{prop:exact_csp}}\label{prf:proof_exact_csp}
\begin{proof}
\hypertarget{prf:proof_exact_csp}{}
By Lemma~\ref{lem:scaling_csp}, the scaling map $\Scale{a(h)}$ is conformally symplectic with factor $a(h)$, while by construction the one-step core map $\Psi_{\theta,h}$ is symplectic for each step size $h$. Applying Lemma~\ref{lem:csp_composition} twice to
\[
\widehat\Phi_{\Theta,h}=\Scale{a(h)} \circ \Psi_{\theta,h} \circ \Scale{a(h)}
\]
yields
\[
(\mathrm{D}\widehat\Phi_{\Theta,h}(z))^\top J\,\mathrm{D}\widehat\Phi_{\Theta,h}(z)=a(h)^2J.
\]
Using the scalar-rate parameterization
\[
a(h)=\exp\!\Bigl(-\frac{\hat\gamma h}{2}\Bigr),
\]
we obtain
\[
a(h)^2=\exp(-\hat\gamma h)=\lambda(h),
\]
which proves \eqref{eq:exact_csp}.
\end{proof}

\subsection{Proof of Theorem~\ref{thm:conjugacy} (Scaling conjugacy)}\label{prf:proof_conjugacy}
\begin{proof}
\hypertarget{prf:proof_conjugacy}{}
Let \(x\in \Scale{a}(U)\). Then there exists \(z\in U\) such that
\(x=\Scale{a}(z)\). Hence \(\Scale{a^{-1}}(x)=z\in U\), and therefore
\[
\Psi_h(x)
=
\Scale{a^{-1}}\circ \Phi_h\circ \Scale{a^{-1}}(x)
\]
is well defined on \(\Scale{a}(U)\).

By Lemma~\ref{lem:scaling_csp}, the scaling map \(\Scale{a^{-1}}\) is conformally symplectic with scalar factor \(a^{-1}\). Since
\(\Phi_h\in \mathrm{CSP}^r_{\lambda(h)}(U)\) and
\(\lambda(h)=a^2\), Lemma~\ref{lem:csp_composition} gives
\[
\bigl(\mathrm{D}\Psi_h(x)\bigr)^\top J\,\mathrm{D}\Psi_h(x)
=
a^{-1}\lambda(h)a^{-1}J
=
a^{-1}a^2a^{-1}J
=
J,
\qquad x\in \Scale{a}(U).
\]
Thus \(\Psi_h\in \mathrm{SP}^r(\Scale{a}(U))\).

Finally, composing the definition of \(\Psi_h\) with \(\Scale{a}\) on the left and on the right gives
\[
\Scale{a}\circ\Psi_h\circ\Scale{a}
=
\Scale{a}\circ\Scale{a^{-1}}\circ\Phi_h\circ\Scale{a^{-1}}\circ\Scale{a}
=
\Phi_h
\]
on \(U\). This proves the factorization.
\end{proof}

\subsection{Proof of Theorem~\ref{thm:density_transfer} (Pointwise-in-step density of CSympNet-ID)}\label{prf:proof_density_transfer}
\begin{proof}
\hypertarget{prf:proof_density_transfer}{}
Fix a compact set \(W\subset U\), a target map
\(\Phi_h\in \mathrm{CSP}^r_{\lambda(h)}(U)\), and \(\varepsilon>0\).
By Theorem~\ref{thm:conjugacy}, with \(a=\sqrt{\lambda(h)}\), the conjugated map
\[
\Psi_h:=\Scale{a^{-1}}\circ \Phi_h\circ \Scale{a^{-1}}
\]
is well defined on \(\Scale{a}(U)\) and belongs to \(\mathrm{SP}^r(\Scale{a}(U))\).
Since \(\Scale{a}(W)\subset \Scale{a}(U)\) is compact, the assumed \(r\)-uniform density of the fixed-\(h\) neural core family gives, for any \(\eta>0\), a core parameter \(\theta\) such that
\[
\left\|\Psi_{\theta,h}-\Psi_h\right\|_{C^r(\Scale{a}(W))}<\eta .
\]
Define
\[
F_{\theta,h}
:=
\Scale{a}\circ\Psi_{\theta,h}\circ\Scale{a}.
\]
By Lemma~\ref{lem:csp_composition}, \(F_{\theta,h}\in
\mathrm{CSP}^r_{\lambda(h)}(U)\), since the two scaling maps contribute the factors
\(a\) and \(a\), while \(\Psi_{\theta,h}\) is symplectic and \(a^2=\lambda(h)\). Moreover, because left and right composition with the fixed linear map \(\Scale{a}\) are bounded operations in the \(C^r\) norm on compact sets, there exists a constant
\(C=C(r,a,W)>0\) such that
\[
\left\|\widehat\Phi_{\theta,\gamma,h}-\Phi_h\right\|_{C^r(W)}
\le
C
\left\|\Psi_{\theta,h}-\Psi_h\right\|_{C^r(\Scale{a}(W))}.
\]
Choosing \(\eta=\varepsilon/C\) yields
\[
\left\|\widehat\Phi_{\theta,\gamma,h}-\Phi_h\right\|_{C^r(W)}<\varepsilon .
\]
Since \(W\subset U\), \(\Phi_h\in \mathrm{CSP}^r_{\lambda(h)}(U)\), and \(\varepsilon>0\) were arbitrary, the CSympNet-ID family is \(r\)-uniformly dense on compacta in
\(\mathrm{CSP}^r_{\lambda(h)}(U)\).
\end{proof}

\subsection{Additional experiment: sample efficiency}\label{app:sample_efficiency}

\begin{figure}[H]
\centering
\begin{subfigure}[t]{0.5\linewidth}
  \centering
  \includegraphics[width=\linewidth]{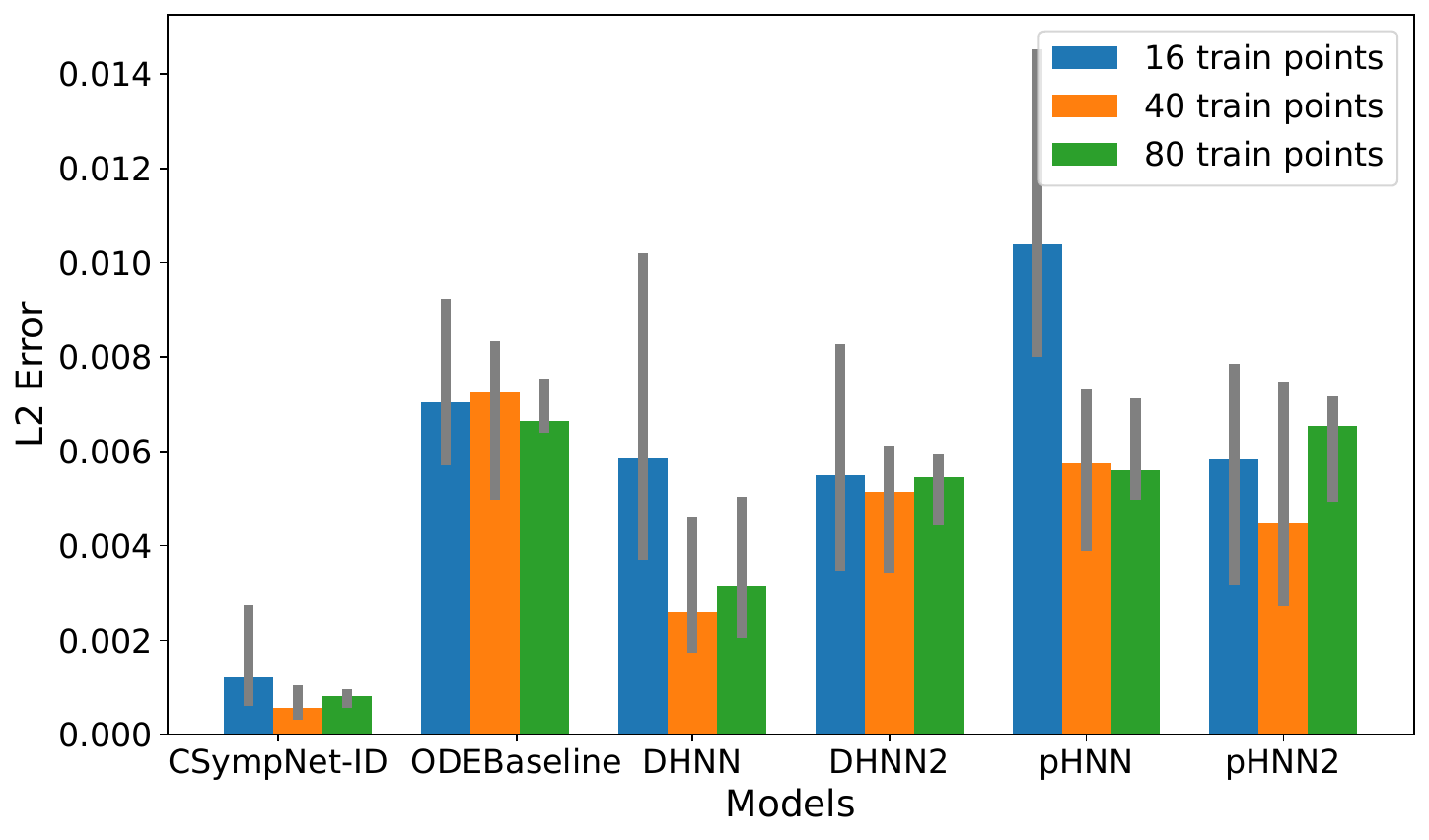}
  \caption{S1: damped harmonic oscillator.}
\end{subfigure}
\begin{subfigure}[t]{0.5\linewidth}
  \centering
  \includegraphics[width=\linewidth]{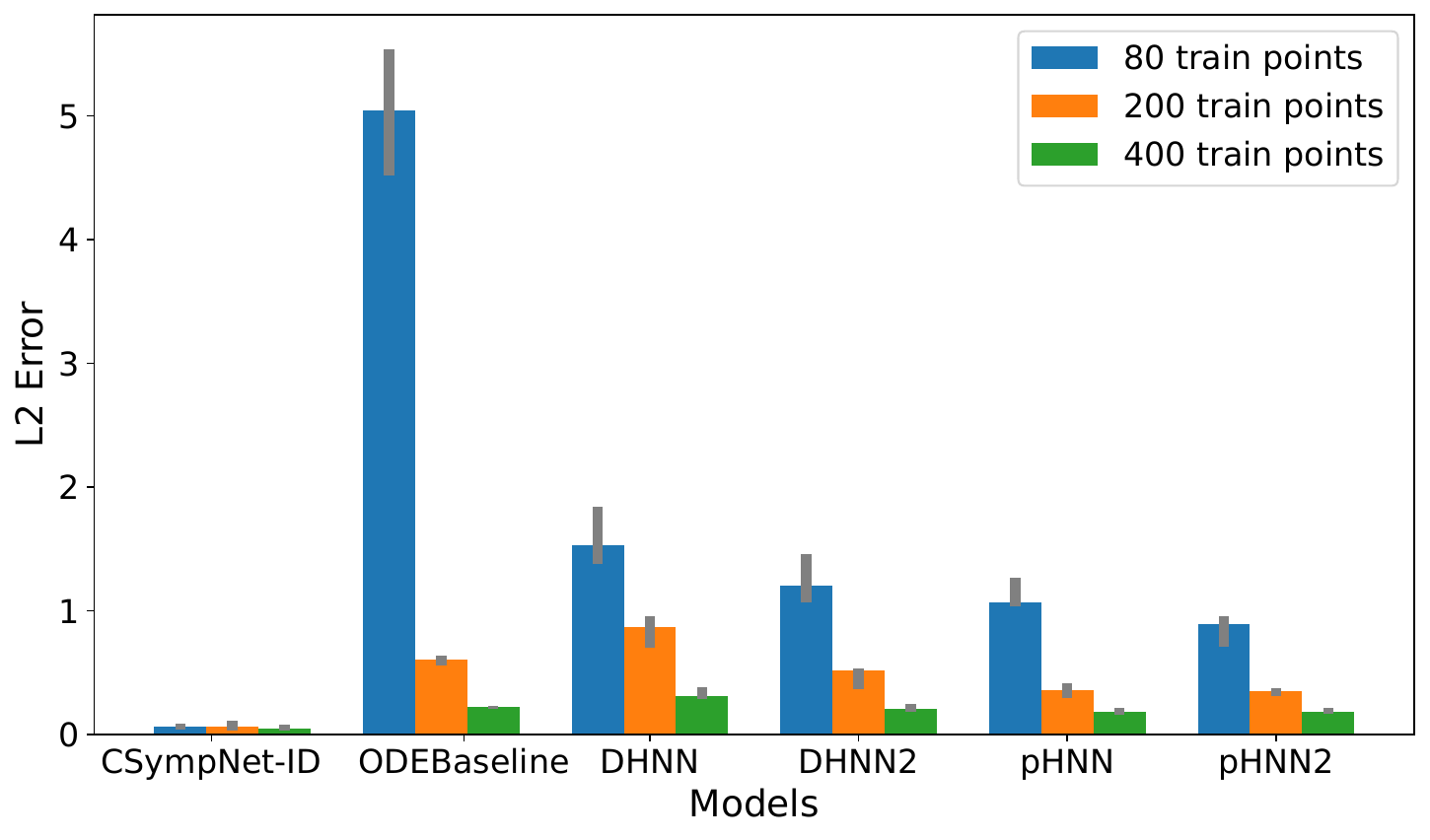}
  \caption{S2: damped spring-mass chain.}
\end{subfigure}
\begin{subfigure}[t]{0.5\linewidth}
  \centering
  \includegraphics[width=\linewidth]{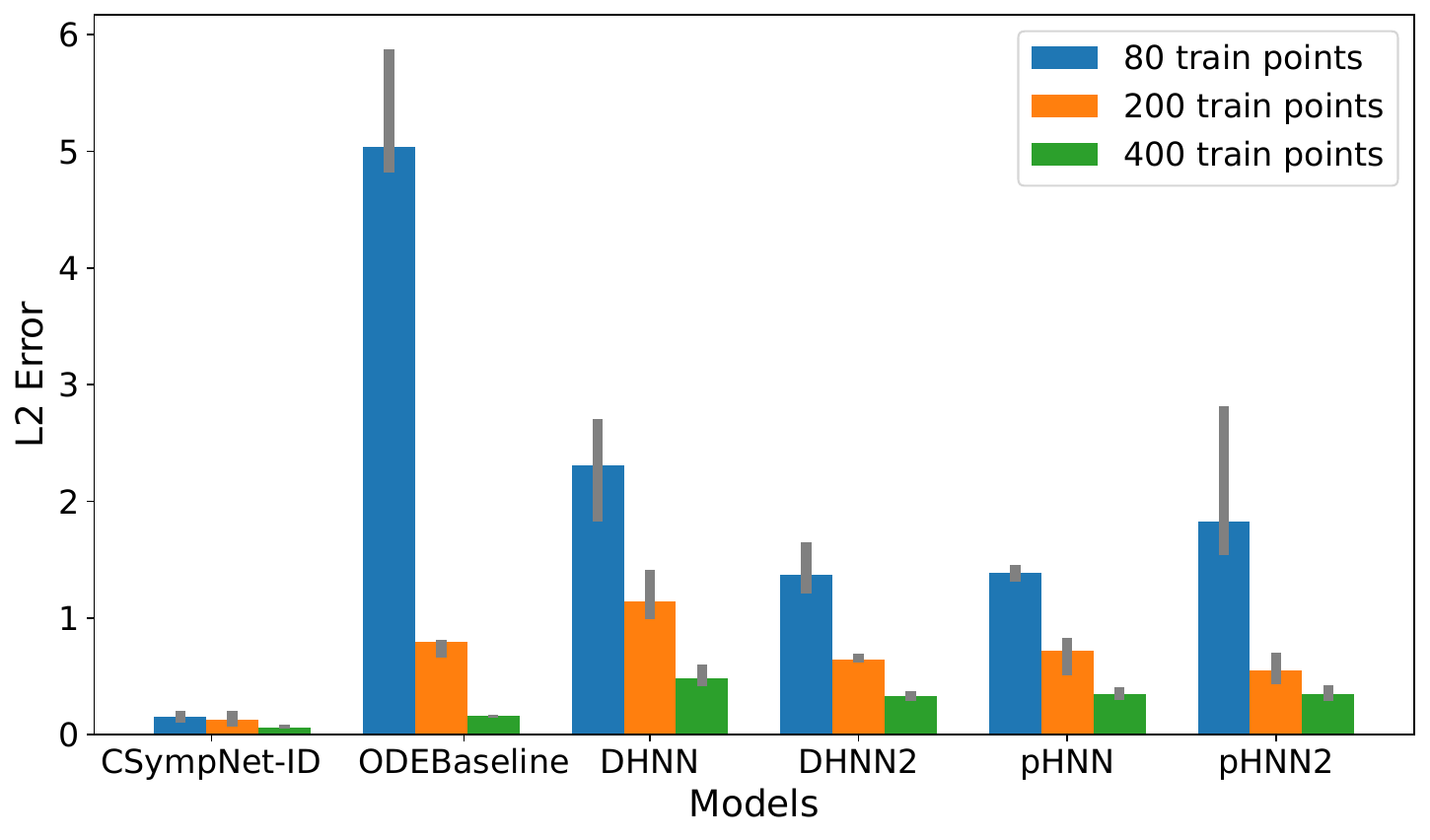}
  \caption{S3: damped nonlinear cubic oscillator.}
\end{subfigure}
\caption{Sample-efficiency comparison at $\sigma=0$. Bars show the median over $10$ random seeds, and error bars show the interquartile range for the training-point settings shown in the legends.}
\label{fig:pointnum_all}
\end{figure}
The auxiliary comparisons in Appendices~\ref{app:sample_efficiency}-\ref{app:heldout_step} focus on CSympNet-ID and the dissipative vector-field baselines.  We report sample-efficiency comparisons in Fig.~\ref{fig:pointnum_all}. SympNet and MapBaseline are not shown in these bar plots since their rollout errors are much larger on the same scale and would obscure the differences among the displayed methods; they are included in the full  comparisons in the main text.
CSympNet-ID remains robust as the number of training pairs is reduced: its rollout error is nearly insensitive to the training-set size and remains the lowest on all test problems. 
These results suggest that the proposed conformal-symplectic parameterization improves data efficiency in the reported tests.

\subsection{Additional experiment: held-out-step sweep}\label{app:heldout_step}

\begin{figure}[H]
\centering
\begin{subfigure}[t]{0.5\linewidth}
  \centering
  \includegraphics[width=\linewidth]{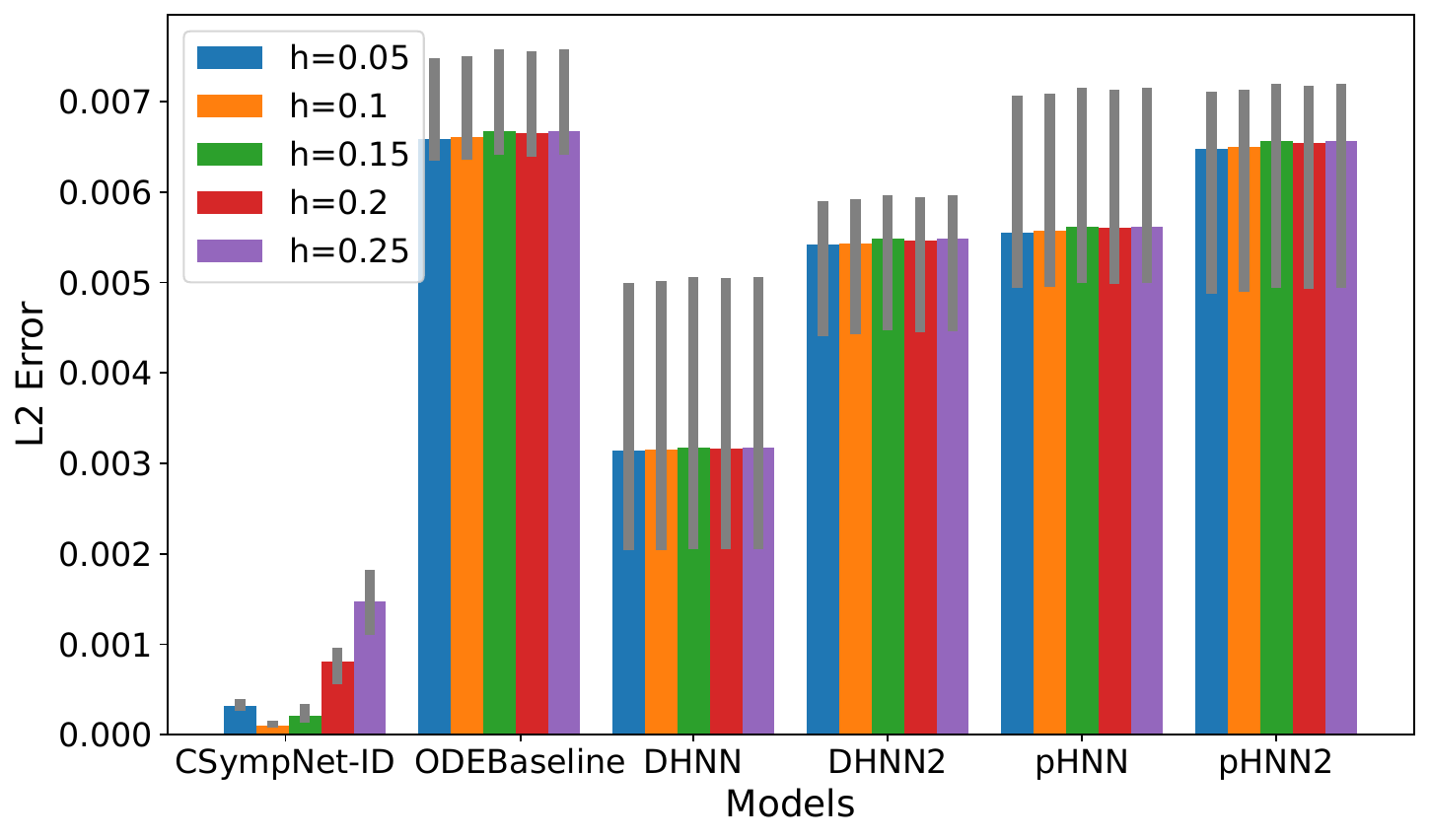}
  \caption{S1, $\sigma=0$.}
\end{subfigure}\hfill
\begin{subfigure}[t]{0.5\linewidth}
  \centering
  \includegraphics[width=\linewidth]{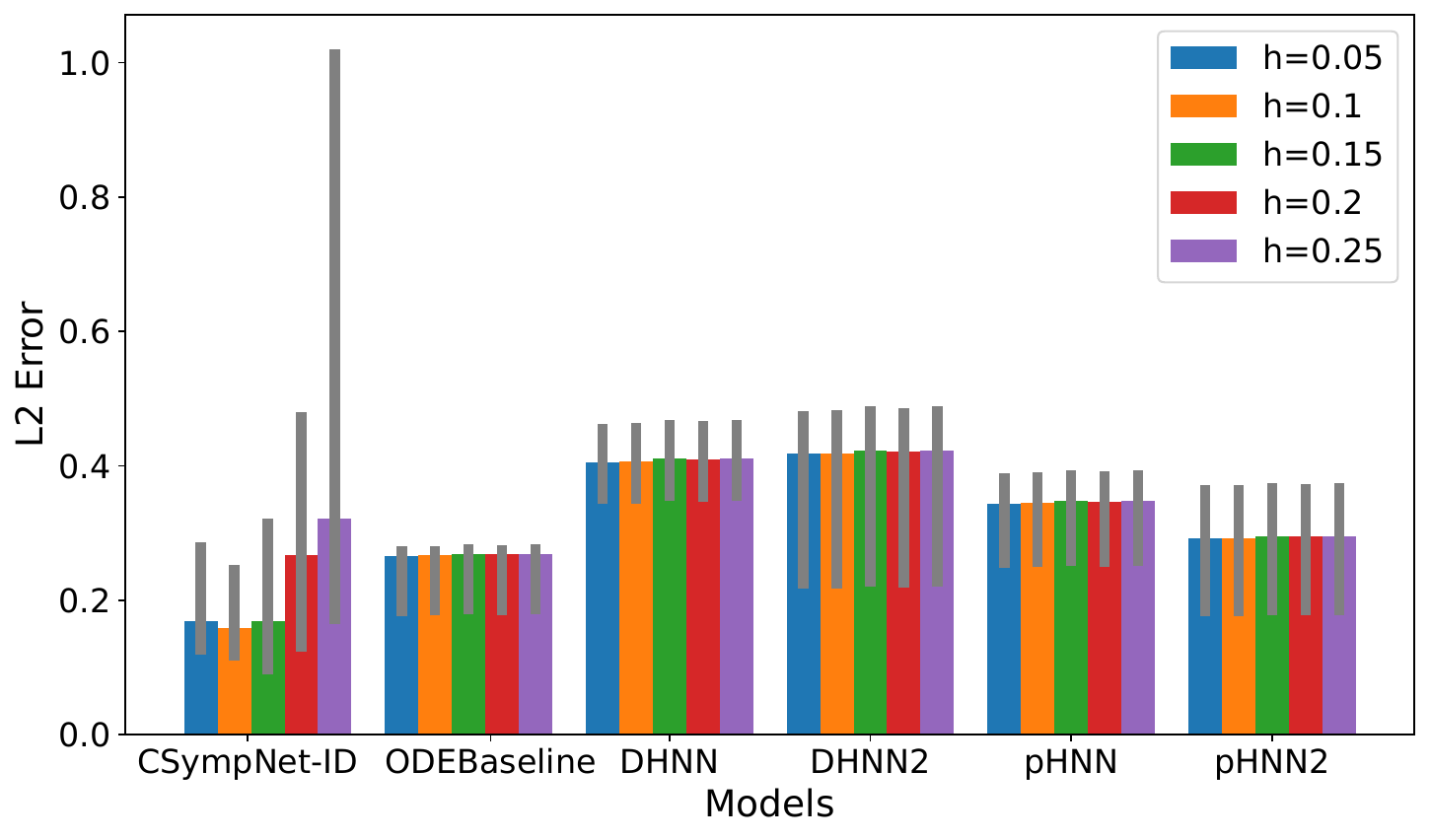}
  \caption{S1, $\sigma=0.01$.}
\end{subfigure}

\begin{subfigure}[t]{0.5\linewidth}
  \centering
  \includegraphics[width=\linewidth]{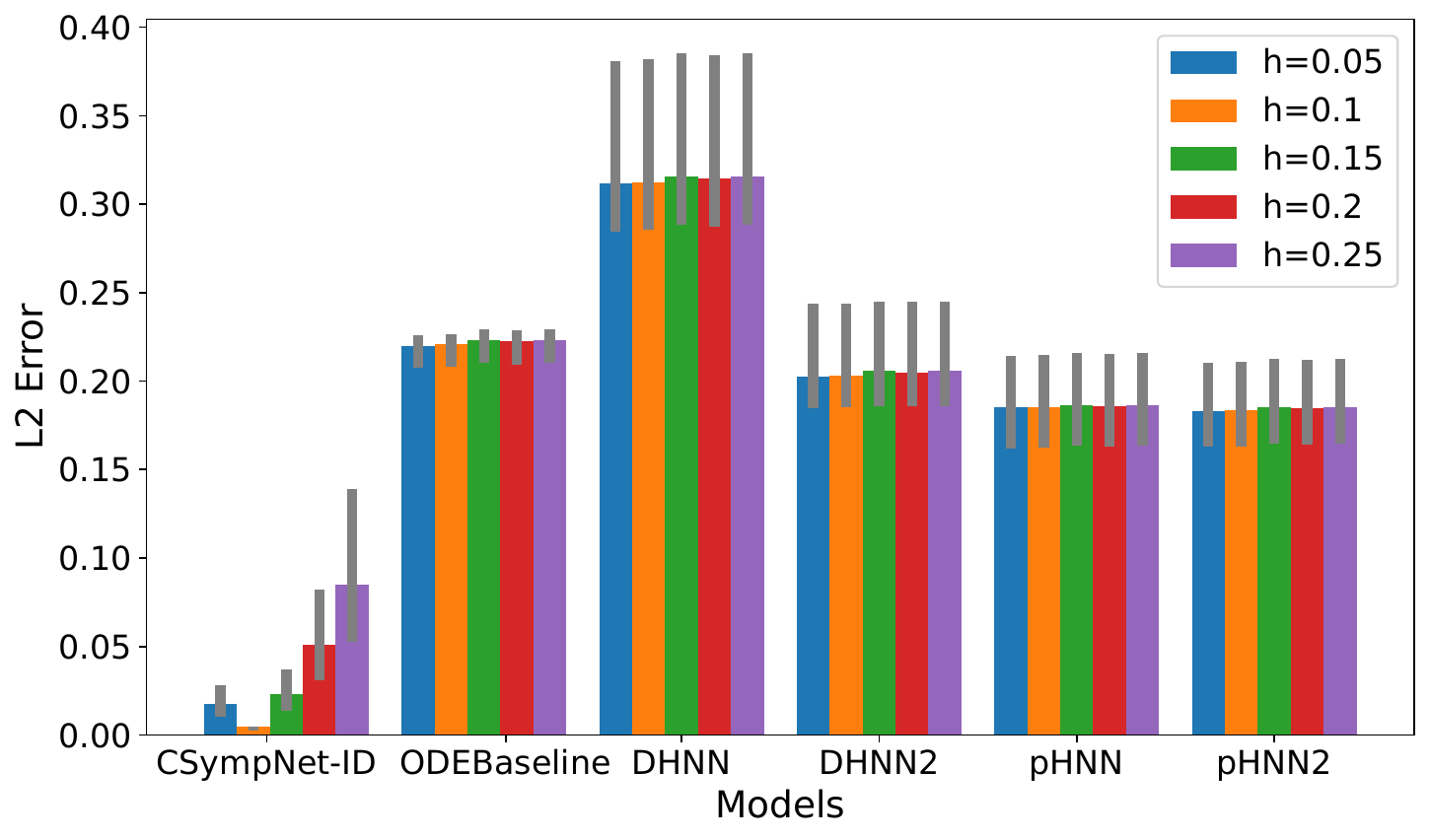}
  \caption{S2, $\sigma=0$.}
\end{subfigure}\hfill
\begin{subfigure}[t]{0.5\linewidth}
  \centering
  \includegraphics[width=\linewidth]{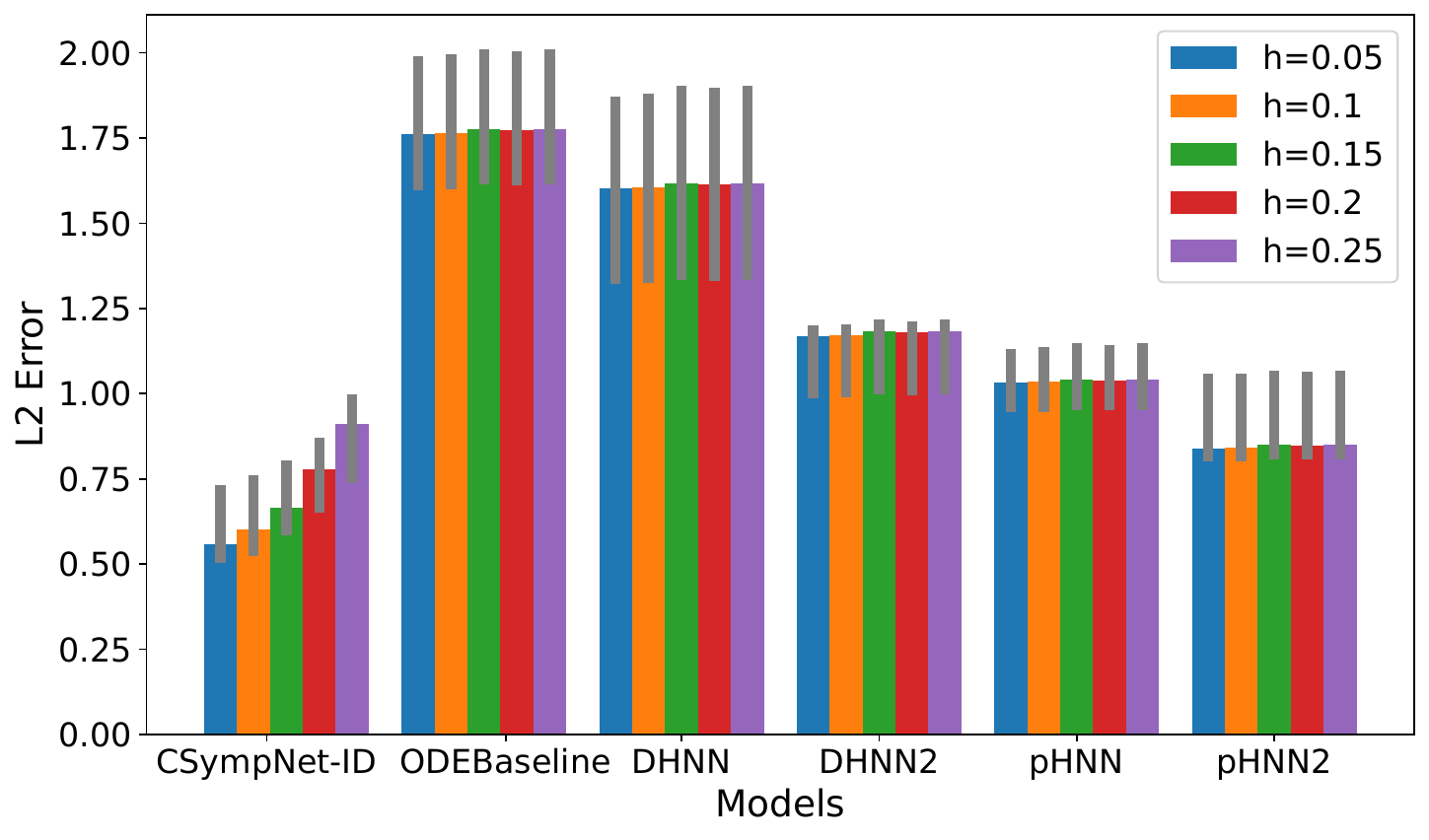}
  \caption{S2, $\sigma=0.01$.}
\end{subfigure}

\begin{subfigure}[t]{0.5\linewidth}
  \centering
  \includegraphics[width=\linewidth]{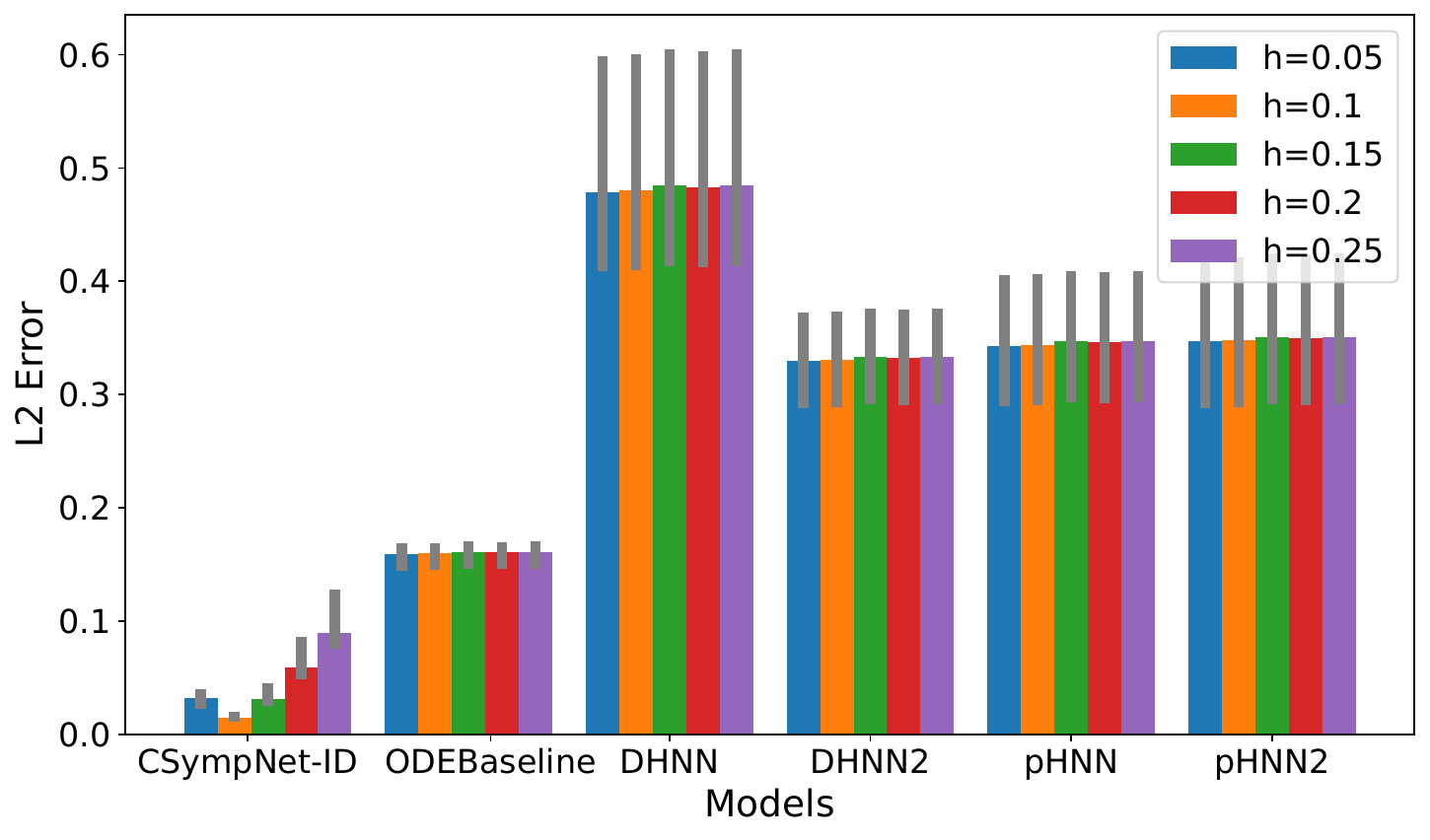}
  \caption{S3, $\sigma=0$.}
\end{subfigure}\hfill
\begin{subfigure}[t]{0.5\linewidth}
  \centering
  \includegraphics[width=\linewidth]{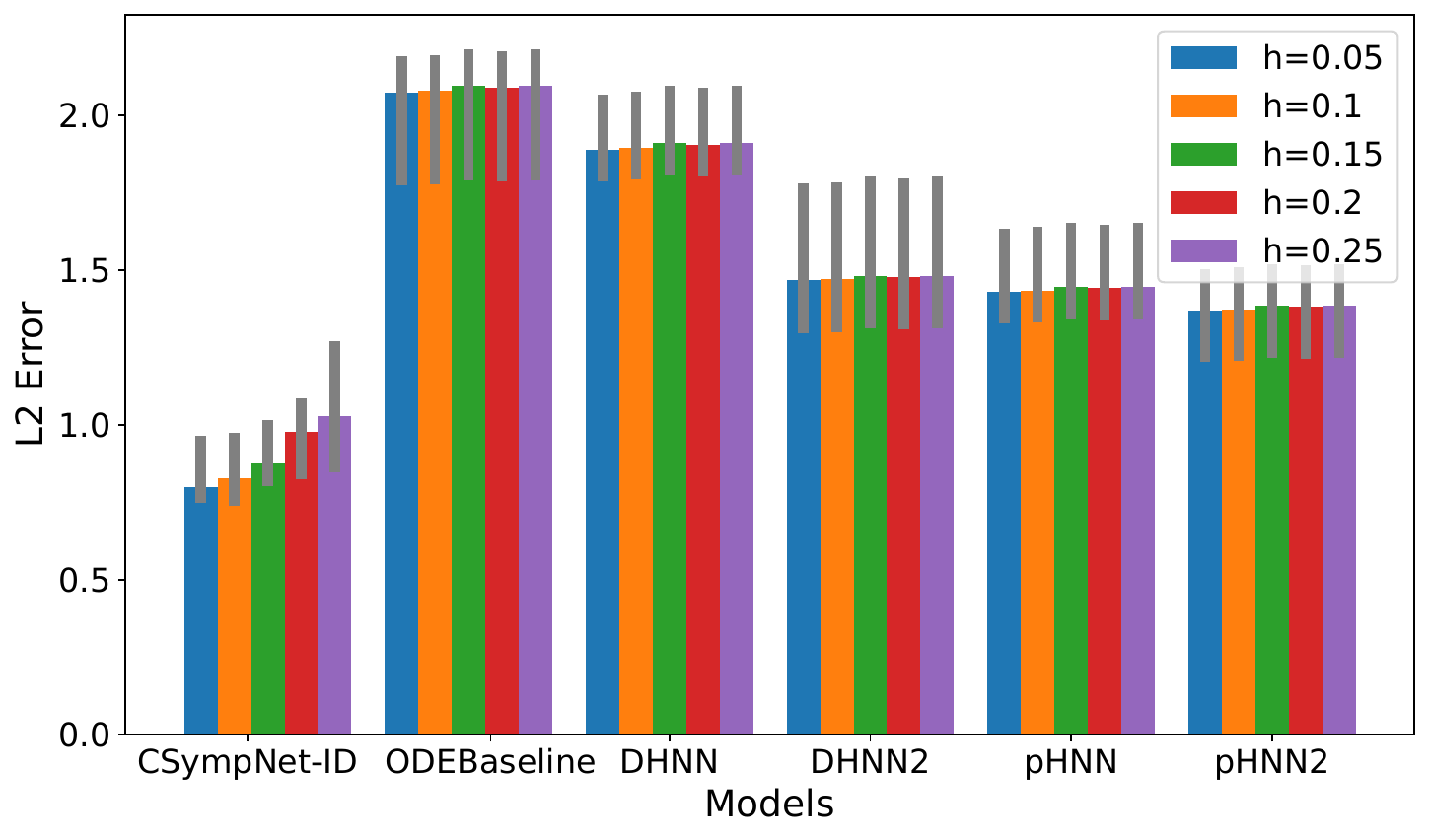}
  \caption{S3, $\sigma=0.01$.}
\end{subfigure}
\caption{Held-out-step sweep for long-horizon rollout error. Bars show the median over 10 random seeds, and error bars show the interquartile range.}
\label{fig:h_sweep_error}
\end{figure}

We report a held-out-step sweep in Fig.~\ref{fig:h_sweep_error}. Models are evaluated under $\sigma=0$ and $\sigma=0.01$, with $h_{\mathrm{test}}\in\{0.05,0.1,0.15,0.2,0.25\}$. On S1, CSympNet-ID remains competitive, but its error and variability increase for larger test steps under noisy observations. On  S2 and  S3, CSympNet-ID gives consistently lower median rollout errors than the compared baselines across the tested step sizes. These results suggest that the proposed step-dependent conformal-symplectic map provides stable cross-step rollout behavior, with the clearest gains appearing in the more challenging benchmarks.

% ===================== References 
\bibliographystyle{elsarticle-num}
\bibliography{paper}
\end{document}